\documentclass{article}

\usepackage{arxiv}

\usepackage[utf8]{inputenc} 
\usepackage[T1]{fontenc}    
\usepackage{hyperref}       
\usepackage{url}            
\usepackage{booktabs}       
\usepackage{amsfonts}       
\usepackage{nicefrac}       
\usepackage{microtype}      
\usepackage{lipsum}
\usepackage{times}
\usepackage[ruled]{algorithm2e}
\usepackage{amsmath,amssymb,amsthm}
\usepackage{natbib}
\setcitestyle{authoryear,open={(},close={)}}
\usepackage{dsfont}
\usepackage{xpatch}
\xpatchcmd{\proof}{\itshape}{\normalfont\proofnamefont}{}{}

\newcommand{\proofnamefont}{\bfseries}

\newtheorem{proposition}{Proposition}
\newtheorem{corollary}[proposition]{Corollary}
\newtheorem{lemma}[proposition]{Lemma}

\newtheorem{theorem}[proposition]{Theorem}

\newtheorem{assumption}{Assumption}

\title{Logarithmic regret for  parameter-free online logistic regression}

\author{%
 Joseph de Vilmarest\\ 
 \texttt{josephdevilmarest@gmail.com}\\
 LPSM, Sorbonne Université\\
 4 Place Jussieu, 75005 Paris, France
\And Olivier Wintenberger\\
 \texttt{olivier.wintenberger@upmc.fr}\\
 LPSM, Sorbonne Université\\
 4 Place Jussieu, 75005 Paris, France
}

\newcommand{\R}{\mathbb{R}}

\begin{document}

\maketitle

\begin{abstract}%
We consider online optimization procedures in the context of logistic regression, focusing on the Extended Kalman Filter (EKF). We introduce a second-order algorithm close to the EKF, named Semi-Online Step (SOS), for which we prove a $\mathcal{O}(\log(n))$ regret in the adversarial setting, paving the way to similar results for the EKF. This regret bound on SOS is the first for such parameter-free algorithm in the adversarial logistic regression. We prove for the EKF in constant dynamics a $\mathcal{O}(\log(n))$ regret in expectation and in the well-specified logistic regression model.
\end{abstract}


\section{Introduction}
In the convex online optimization literature \citep{hazan2016introduction}, a crucial issue is the tuning of parameters. Our aim is to develop parameter-free algorithms in the context of logistic regression. One observes $y_t\in \{-1,1\}$ recursively through time $t=1,2,\ldots$. At each instance $t-1$, the objective is to construct a prediction of the next value $y_{t}$. In hand, we have explanatory variables $X_t$ in $\R^d$ along with the past pairs $(X_s,y_s)_{s<t}$. We reduce the prediction to a $d$-dimensional optimization problem thanks to the logistic loss function
$$
\ell_{t}(y_{t},\hat \theta_{t})=\log(1+ \exp(-y_t \hat \theta_t^T X_t)),  \qquad t=1,2,\ldots,
$$
where $\hat \theta_t\in \R^d$, $t=1,2,\ldots$, are provided by a recursive algorithm. 
The aim of online convex optimization is to provide regret bounds on the cumulative losses
$\sum_{t=1}^n\ell_t(y_t,\hat \theta_t)$ for algorithms whose recursive update step is of constant complexity.

The logistic loss is exp-concave, property that guarantees the existence of online procedures achieving $O(\log(n))$ regret in the adversarial setting. The seminal paper of  \cite{hazan2007logarithmic} proposed two such algorithms, Online Newton Step and Follow The Approximate Leader, that achieve this rate of convergence. Both methods require the knowledge of some constants unknown in practice, namely the constant of exp-concavity and an upper-bound on the gradients of the losses. They also require a projection step on a convex set of finite diameter. We consider these methods as localized ones because they use the strongly convex paraboloid local approximation of any exp-concave functions stated in Lemma 3 in \cite{hazan2007logarithmic}.

On the contrary, some recent papers \citep{bach2013non,gadat2017optimal,godichon2018lp} propose global algorithms in the stochastic setting.  \cite{bach2013non} provide sharp regret bounds for a two-step procedure where the crucial step is the averaging of a Stochastic Gradient Descent (SGD) with constant learning rate that has to be tuned. In \cite{gadat2017optimal,godichon2018lp} the authors propose  non-asymptotic regret bounds with large constants of the averaging of a SGD with more robust learning rates that does not need to be tuned. Our results have the same flavor on a very popular online algorithm, the Extended Kalman Filter (EKF), whose non asymptotic properties have not yet been studied.

For linear regression, Kalman filters as originally described in \cite{kalman1961new} present a Bayesian perspective. The idea is to estimate the conditional expectation of the future state and its variance, given a prior on the initial state and past observations that follow a dynamic model.  Kalman recursion is exactly the ridge regression estimator, see \cite{diderrich1985kalman}, so Kalman filter achieves a $\mathcal{O}(\log(n))$ regret for quadratic losses in adversarial setting. Note that the global strong convexity of the loss is crucial in the analysis of the regret in \cite{cesa2006prediction}.

The Extended Kalman Filter (EKF) of \cite{fahrmeir1992posterior} yields an online parameter-free algorithm for logistic regression. More generally, EKF works in any misspecified Generalized Linear Model as defined in \cite{rigollet2012kullback}. Recently, the equivalence between Kalman filtering under constant dynamics and Online Natural Gradient has been noticed by \cite{ollivier2018online}. It is our belief that Kalman filtering offers an optimal way to choose the step-size in an online gradient descent algorithm. Up to our knowledge, regret bounds have been derived for the batch Maximum Likelihood Estimator only, also called Follow The Leader in the online learning literature. The complexity of this batch algorithm is prohibitive, see the discussion in \cite{hazan2007logarithmic}. In our paper, we view the EKF as an approximation of FTL in order to derive a $\mathcal{O}(\log(n))$ regret bound.

As an intermediate step, we prove a $\mathcal{O}(\log(n))$ regret in the logistic regression problem for a second-order algorithm between FTL and EKF. We name it the Semi-Online Step (SOS) algorithm as it requires $t$ computations at each step, i.e. its complexity is quadratic in the number of iterations. Despite its inefficiency, SOS analysis is interesting as the non-asymptotic guarantee is valid in any adversarial setting. One can also interpret the extra $t$ computations per iteration compared to the EKF as the cost of the estimation of the local strong convexity constant of the paraboloid approximation.

The EKF is the natural online approximation of the SOS. It is efficient (constant time per iteration) and we prove a $\mathcal{O}(\log(n))$ regret, in expectation and in the well-specified logistic regression setting only. The analysis of the regret splits in two steps. When the algorithm is close to the optimum, its regret is logarithmic with high probability. This logarithmic rate  is due to the nice martingale properties of the gradients of the losses. The conditional expectation of the gradient is proportional to its quadratic variation. The logarithmic regret bound  follows from the local paraboloid approximation of \cite{hazan2007logarithmic}. The other phase, when the algorithm explores the optimization space,   is much more problematic to analyze because the local paraboloid approximation does not apply uniformly. To circumvent this issue, we appeal at  more robust potential arguments as in \cite{gadat2017optimal}. We got a logarithmic control on the number of iterations spent in the first phase in expectation only. It is an open question whether this number of iterations can be controlled with high probability.

The paper is organized as follows. In Section \ref{SOS}, we introduce the SOS algorithm and we give its $\mathcal{O}(\log(n))$ regret in Theorem \ref{regretsemionline} followed by its proof. In Theorem   \ref{resultKalman} of Section \ref{EKF}, we present our result in expectation for the EKF. We present the main steps of the  proof of  Theorem   \ref{resultKalman} in Section \ref{proofEKF}. Finally we discuss the results and future work in Section \ref{future}.

\section{Semi-Online Step algorithm}
\label{SOS}
In Section \ref{sec:sos}, we introduce the SOS algorithm as a semi-online approximation of the batch FTL algorithm
\begin{equation}\label{eq:ftl}
    \theta_t^* \in \arg\min_{\theta} \sum\limits_{s=1}^{t-1} l_s(y_s,\theta)\,.
\end{equation}
We see in Section \ref{sec:compl} that SOS is also very close to the EKF but with higher complexity.   Then we prove a bound on the regret of SOS in Section \ref{sec:regsos}.
\begin{algorithm}[t]
\label{alg:sos}
{\caption{Semi-Online Steps}}
{
\begin{enumerate}
\item {\it Initialization}: $\tilde{P}_1$ is any positive definite matrix, $\tilde{\theta}_1$ is any initial parameter in $\R^d$.
\item {\it Iteration}: at each time step $t=1,2,\ldots$
\begin{enumerate}
\item Compute the matrix $\tilde{P}_{t+1} = \left(\tilde{P}_1^{-1} + \sum\limits_{s=1}^t \frac{X_sX_s^T}{(1+e^{\tilde{\theta}_t^TX_s})(1+e^{-\tilde{\theta}_t^TX_s})}\right)^{-1}$:\\
Starting from $\tilde{P}_{t+1}^{(0)} = \tilde{P}_1$, we compute\\
$\tilde{P}_{t+1}^{(u)} = \left(\tilde{P}_1^{-1} + \sum\limits_{s=1}^u \frac{X_sX_s^T}{(1+e^{\tilde{\theta}_t^TX_s})(1+e^{-\tilde{\theta}_t^TX_s})}\right)^{-1}$ thanks to the recursion
$$
    \tilde{P}_{t+1}^{(u)}  = \tilde{P}_{t+1}^{(u-1)}  - \frac{\tilde{P}_{t+1}^{(u-1)}X_uX_u^T\tilde{P}_{t+1}^{(u-1)}}{1+X_u^T\tilde{P}_{t+1}^{(u-1)}X_u\tilde{p}_t^{(u)}(1-\tilde{p}_t^{(u)})}\tilde{p}_t^{(u)}(1-\tilde{p}_t^{(u)})\,,
$$
with $\tilde{p}_t^{(u)} = 1/(1+e^{-\tilde{\theta}_t^TX_u})$ for any $u=1,\ldots,t$, so that $\tilde{P}_{t+1} = \tilde{P}_{t+1}^{(t)}$\,.
\item Update 
$$
 \tilde{\theta}_{t+1} = \tilde{\theta}_t + \tilde{P}_{t+1} \frac{y_tX_t}{1+e^{y_t\tilde{\theta}_t^TX_t}}\,.
$$
\end{enumerate}
\end{enumerate}
}
\end{algorithm}
\subsection{Construction of the SOS algorithm}\label{sec:sos}
The Semi-Online Step  is described in Algorithm \ref{alg:sos}. 
We derive it from the Taylor approximation
$$
    \frac{\partial}{\partial \theta}\left[\sum\limits_{s=1}^t l_s(y_s,\theta)\right] \approx \frac{\partial}{\partial \theta}\left[\sum\limits_{s=1}^t l_s(y_s,\theta)\right]\Bigr|_{\substack{\theta=\theta_t^*}} + \frac{\partial^2}{\partial \theta^2}\left[\sum\limits_{s=1}^t l_s(y_s,\theta)\right]\Bigr|_{\substack{\theta=\theta_t^*}}(\theta-\theta_t^*)\,,
$$
which transforms the first order condition of the optimization problem \eqref{eq:ftl} realized by $\theta_{t+1}^*$  into
$$
    \frac{\partial}{\partial \theta}\left[\sum\limits_{s=1}^t l_s(y_s,\theta)\right]\Bigr|_{\substack{\theta=\theta_t^*}} + \frac{\partial^2}{\partial \theta^2}\left[\sum\limits_{s=1}^t l_s(y_s,\theta)\right]\Bigr|_{\substack{\theta=\theta_t^*}}(\theta_{t+1}^*-\theta_t^*) \approx 0\,.
$$
Using the definition of $\theta_t^*$ we have
$$
 \frac{\partial}{\partial \theta}\left[\sum\limits_{s=1}^{t-1} l_s(y_s,\theta_t^*)\right]=0\,.
$$
Combining this identity and the definition of the derivatives of the logistic loss we obtain
$$
    \frac{\partial}{\partial \theta}\left[\sum\limits_{s=1}^t l_s(y_s,\theta)\right]\Bigr|_{\substack{\theta=\theta_t^*}} = \frac{\partial}{\partial \theta} l_t(y_t,\theta)\Bigr|_{\substack{\theta=\theta_t^*}} = \frac{-y_tX_t}{1+e^{y_t\theta_t^{*T}X_t}}\,,
$$
$$
    \frac{\partial^2}{\partial \theta^2}\left[\sum\limits_{s=1}^t l_s(y_s,\theta)\right]\Bigr|_{\substack{\theta=\theta_t^*}} = \sum\limits_{s=1}^t \frac{X_sX_s^T}{(1+e^{\theta_t^{*T}X_s})(1+e^{-\theta_t^{*T}X_s})}\,.
$$
Therefore $\theta_{t+1}^*$ satisfies approximately 
$$
    \left(\sum\limits_{s=1}^t \frac{X_sX_s^T}{(1+e^{\theta_t^{*T}X_s})(1+e^{-\theta_t^{*T}X_s})}\right)(\theta_{t+1}^* - \theta_t^*) \approx \frac{y_tX_t}{1+e^{y_t\theta_t^{*T}X_t}}\,.
$$
If the Hessian matrix were invertible, we would obtain 
$$
    \theta_{t+1}^* \approx \theta_t^* + \left(\sum\limits_{s=1}^t \frac{X_sX_s^T}{(1+e^{\theta_t^{*T}X_s})(1+e^{-\theta_t^{*T}X_s})}\right)^{-1} \frac{y_tX_t}{1+e^{y_t\theta_t^{*T}X_t}}\,.
$$
This relation approximately satisfied by the optima sequence  $(\theta_t^*)$ motivates the introduction of the SOS algorithm as defined in Algorithm \ref{alg:sos}. The computation of $\tilde{P}_{t+1}$ relies on the Sherman-Morrison formula: if $A\in\R^{d\times d}$ and $u,v\in\R^d$, 
\begin{equation}\label{eq:smformula}
    \left(A+uv^T\right)^{-1} = A^{-1} - \frac{A^{-1}uv^TA^{-1}}{1+v^TA^{-1}u}\,.
\end{equation}
We introduce the regularization matrix $\tilde{P}_1$ which guarantees the positive definiteness of $\tilde{P}_{t}$ in  Algorithm \ref{alg:sos}.  A good choice is for instance  $\tilde{P}_1^{-1} = \frac{I}{p_1}$, $p_1>0$. SOS then corresponds to the approximation
$$
    \tilde{\theta}_t \approx \arg\min\limits_{\theta} \left(\sum\limits_{s=1}^{t-1} l_s(y_s,\theta) + \frac{1}{2p_1}\|\theta\|^2 \right),\qquad t=1,2,\ldots.
$$

\subsection{Comparison with EKF}\label{sec:compl}
The Extended Kalman Filter was introduced by \cite{fahrmeir1992posterior} for any Dynamic Generalized Linear Model. For constant dynamics, the EKF is shown to be equivalent to the Online Natural Gradient algorithm in \cite{ollivier2018online}, yielding the recursion
\begin{align*}
P_{t+1}^{-1}&=P_t^{-1}+\frac{X_tX_t^T}{(1+e^{\hat{\theta}_t^TX_t})(1+e^{-\hat{\theta}_t^TX_t})}\,,\\
\hat{\theta}_{t+1} &= \hat{\theta}_t - P_{t+1}  \frac{\partial}{\partial \theta} l_t(y_t,\theta)\Bigr|_{\substack{\theta=\hat \theta_t}}\,.
\end{align*}
This EKF recursion departs from SOS in the update of the matrix $P_t$ which satisfies
$$
    P_{t+1} = \left(P_1^{-1} + \sum\limits_{s=1}^t \frac{X_sX_s^T}{(1+e^{\hat{\theta}_s^TX_s})(1+e^{-\hat{\theta}_s^TX_s})}\right)^{-1},\qquad t=1,2,\ldots.
$$
In EKF, we add a rank-one matrix to get $P_{t+1}^{-1}$ from $P_t^{-1}$ in order to update the matrix efficiently. On the contrary,  the matrix $\tilde{P}_t$ in SOS is recomputed at each step because the Hessian has to be computed at the current estimate $\tilde{\theta}_t$. Despite the similarity between $P_t$ and $\tilde{P}_t$ we were not able to control their differences. Our analysis of SOS and EKF are distinct and the obtained regret bounds are different in nature.

\begin{algorithm}[t]
\label{alg:ekf}
{\caption{Extended Kalman Filter}}
{
\begin{enumerate}
\item {\it Initialization}: $P_1$ is any positive definite matrix, $\hat{\theta}_1$ is any initial parameter in $\R^d$.
\item {\it Iteration}: at each time step $t=1,2,\ldots$
\begin{enumerate}
\item Update
$$
    P_{t+1}  = P_t  - \frac{P_tX_tX_t^TP_t}{1+X_t^TP_tX_t\hat{p}_t(1-\hat{p}_t)}\hat{p}_t(1-\hat{p}_t)\,,
$$
with $\hat{p}_t = 1/(1+e^{-\hat{\theta}_t^TX_t})$.
\item Update 
$$
  \hat{\theta}_{t+1} = \hat{\theta}_t + P_{t+1} \frac{y_tX_t}{1+e^{y_t\hat{\theta}_t^TX_t}}\,.
$$
\end{enumerate}
\end{enumerate}
}
\end{algorithm}
Thanks to the Sherman-Morrison formula \eqref{eq:smformula}, we describe the EKF in Algorithm \ref{alg:ekf} avoiding any inversion of matrices.
The spatial complexity of the two algorithms is $\mathcal{O}(d^2)$ due to the storage of the matrices $P_{t+1}$ and $\tilde{P}_{t+1}$. In term of running time, at each step of the SOS algorithm we have to compute recursively $\tilde{P}_{t+1}^{(u)}$ for $u=1,\ldots, t$  and then $\tilde{\theta}_{t+1}$. Each recursion on $\tilde{P}_{t+1}^{(u)}$ in $u$ requires the computation of a rank-one matrix (product vector-vector) and its addition to the sum, its complexity is $\mathcal{O}(d^2)$. Thus, the complexity of step $t$ in SOS is $\mathcal{O}(td^2)$. As a comparison, the EKF updates $P_t$ online and therefore requires only $\mathcal{O}(d^2)$ operations at each step.

\subsection{The regret bound for SOS and its proof}\label{sec:regsos}
In what follows, we denote
\begin{equation*}
    D_X=\max\limits_{1\le t\le n} \|X_t\|,\ D_{\theta}=\max\limits_{1\le t\le n} \|\tilde{\theta}_t\|,\ D=\max\limits_{1\le t\le n} |\tilde{\theta}_t^TX_t|.
\end{equation*}
SOS offers the advantage to be easier to analyse than EKF. We  prove a $\mathcal{O}(\log(n))$ regret bound on SOS in Theorem \ref{regretsemionline}. Note that the leading constant is the inverse square of the exp-concavity constant times $d^{3/2}D_X(D_\theta+\|\theta\|)$. The localized algorithms of \cite{hazan2007logarithmic} satisfy finer regret bounds with the inverse of the exp-concavity constant times $d$ as the leading constant. We believe that Theorem \ref{regretsemionline} could be improved to get a constant proportional to the inverse of the exp-concavity constant instead of the square inverse, see the end of the proof of Lemma \ref{lemmaSincrement} where we use a very loose bound bringing a $(1+e^D)/2$. Up to our knowledge, SOS is the first parameter-free algorithm that achieves a $\mathcal{O}(\log(n))$ regret bound in the adversarial logistic regression setting.
\begin{theorem}
\label{regretsemionline}
Starting from $\tilde{P}_1=p_1 I$ and $\tilde{\theta}_1\in \R^d$, for any $(X_t,y_t)_{1\le t\le n}$ and $\theta\in\R^d$, the SOS algorithm achieves the regret bound 
\begin{multline*}
    \sum\limits_{t=1}^n \left(l_t(y_t,\tilde{\theta}_t) - l_t(y_t,\theta)\right) \le \left(\frac{\sqrt{d}D_X(D_{\theta}+\|\theta\|)\left(1+e^D\right)}{4}+1\right)\frac{1+e^D}{2}d\log(1+(n-1)p_1D_X^2)\\ +\frac{\|\tilde{\theta}_1\|^2+\|\theta\|^2}{2p_1}+ D_X(D_{\theta}+\|\theta\|)\,, \qquad n\ge 1.
\end{multline*}
\end{theorem}

\begin{proof}
We first apply a telescopic sum argument
\begin{align*}
    \label{telescopic}
    \sum\limits_{t=1}^n \left(l_t(y_t,\tilde{\theta}_t) - l_t(y_t,\theta)\right) & = \sum\limits_{t=1}^n \left(\sum\limits_{s=1}^t l_s(y_s,\tilde{\theta}_t) - \sum\limits_{s=1}^{t-1} l_s(y_s,\tilde{\theta}_t) - l_t(y_t,\theta)\right) \\
    & = \sum\limits_{t=1}^{n-1} \left(\sum\limits_{s=1}^t l_s(y_s,\tilde{\theta}_t) - \sum\limits_{s=1}^t l_s(y_s,\tilde{\theta}_{t+1})\right) + \sum\limits_{s=1}^n \left(l_s(y_s,\tilde{\theta}_n) - l_s(y_s,\theta)\right)\\
     & = \sum\limits_{t=1}^{n-1} \left(\sum\limits_{s=1}^t l_s(y_s,\tilde{\theta}_t) + \frac12 \tilde{\theta}_t^T\tilde{P}_1^{-1}\tilde{\theta}_t- \sum\limits_{s=1}^t l_s(y_s,\tilde{\theta}_{t+1})-    \frac12 \tilde{\theta}_{t+1}^T\tilde{P}_1^{-1}\tilde{\theta}_{t+1}\right)\\
     & \quad + \sum\limits_{s=1}^n l_s(y_s,\tilde{\theta}_n) +  \frac12\tilde{\theta}_n^T\tilde{P}_1^{-1}\tilde{\theta}_n - \sum\limits_{s=1}^n l_s(y_s,\theta)- \frac12 \theta^T\tilde{P}_1^{-1}\theta\\
     &\quad  +\frac12\theta^T\tilde{P}_1^{-1}\theta -\frac12\tilde{\theta}_1^T\tilde{P}_1^{-1}\tilde{\theta}_1\,.
\end{align*}
Then, defining $S_t(\theta) = \frac{\partial}{\partial \theta} \left[\sum\limits_{s=1}^{t-1} l_s(y_s,\theta)+\frac12  \theta^T\tilde{P}_1^{-1}\theta \right]$, we use the convexity of $S_t+l_t$ to obtain linear bounds:
\begin{align*}
\sum\limits_{t=1}^n \left(l_t(y_t,\tilde{\theta}_t) - l_t(y_t,\theta)\right) 
     \le& \sum\limits_{t=1}^{n-1} \left(S_t(\tilde{\theta}_t) + \frac{\partial l_t(y_t,\theta)}{\partial \theta}\Bigr|_{\substack{\tilde{\theta}_t}}\right)^T(\tilde{\theta}_t-\tilde{\theta}_{t+1})\\
     & + \left(S_n(\tilde{\theta}_n)+ \frac{\partial l_n(y_n,\theta)}{\partial \theta}\Bigr|_{\substack{\tilde{\theta}_n}}\right)^T(\tilde{\theta}_n-\theta)\\
     & +\frac12\theta^T\tilde{P}_1^{-1}\theta -\frac12\tilde{\theta}_1^T\tilde{P}_1^{-1}\tilde{\theta}_1\,.
\end{align*}
We apply another telescopic argument in order to get
\begin{equation*}
    \sum\limits_{t=1}^{n-1} S_t(\tilde{\theta}_t)^T(\tilde{\theta}_t-\tilde{\theta}_{t+1}) = \sum\limits_{t=1}^{n-1} \left(S_{t+1}(\tilde{\theta}_{t+1})-S_t(\tilde{\theta}_t)\right)^T\tilde{\theta}_{t+1} + S_1(\tilde{\theta}_1)^T\tilde{\theta}_1 - S_n(\tilde{\theta}_n)^T\tilde{\theta}_n\,.
\end{equation*}
As $S_1(\tilde{\theta}_1)=\tilde{P}_1^{-1}\tilde{\theta}_1$, we sum up our findings to achieve the regret bound
\begin{multline}\label{eq:regretinterm}
    \sum\limits_{t=1}^n \left(l_t(y_t,\tilde{\theta}_t) - l_t(y_t,\theta)\right) \le \sum\limits_{t=1}^{n-1} \left(S_{t+1}(\tilde{\theta}_{t+1})-S_t(\tilde{\theta}_t)\right)^T\tilde{\theta}_{t+1} - S_n(\tilde{\theta}_n)^T\theta+\frac12\tilde{\theta}^T_1\tilde{P}_1^{-1}\tilde{\theta}_1 +\frac12\theta^T\tilde{P}_1^{-1}\theta \\
    + \sum\limits_{t=1}^{n-1} \left(\frac{\partial l_t(y_t,\theta)}{\partial \theta}\Bigr|_{\substack{\tilde{\theta}_t}}\right)^T(\tilde{\theta}_t-\tilde{\theta}_{t+1}) + \left(\frac{\partial l_n(y_n,\theta)}{\partial \theta}\Bigr|_{\substack{\tilde{\theta}_n}}\right)^T(\tilde{\theta}_n-\theta)\,.
\end{multline}
Next we use the following Lemma proved in Appendix \ref{App1}.
\begin{lemma}
\label{lemmaSincrement}
For any $t=1,2,\ldots$, we have
\begin{equation*}
    \left\|S_{t+1}(\tilde{\theta}_{t+1}) - S_t(\tilde{\theta}_t) \right\| \le \frac{\sqrt{d}D_X\left(1+e^D\right)}{4}\frac{X_t^T \tilde{P}_{t+1}X_t}{(1+e^{y_t\tilde{\theta}_t^TX_t})^2}\,.
\end{equation*}
\end{lemma}
Applying Lemma \ref{lemmaSincrement} on the norm of the first term in the previous regret bound \eqref{eq:regretinterm}, we get
\begin{align*}
    \left\|\sum\limits_{t=1}^{n-1} \left(S_{t+1}(\tilde{\theta}_{t+1})-S_t(\tilde{\theta}_t)\right)^T\tilde{\theta}_{t+1}\right\| & \le \sum\limits_{t=1}^{n-1} \left\|S_{t+1}(\tilde{\theta}_{t+1})-S_t(\tilde{\theta}_t)\right\|\|\tilde{\theta}_{t+1}\| \\
    & \le \frac{\sqrt{d}D_XD_{\theta}\left(1+e^D\right)}{4} \sum\limits_{t=1}^{n-1} \frac{X_t^T \tilde{P}_{t+1}X_t}{(1+e^{y_t\tilde{\theta}_t^TX_t})^2}\,.
\end{align*}
Similarly, we estimate the second term of the regret bound \eqref{eq:regretinterm} as
\begin{equation*}
    \left\|S_n(\tilde{\theta}_n)^T\theta\right\|  \le \sum\limits_{t=1}^{n-1} \left\|S_{t+1}(\tilde{\theta}_{t+1})-S_t(\tilde{\theta}_t)\right\|\|\theta\|  \le \frac{\sqrt{d}D_X\|\theta\|\left(1+e^D\right)}{4} \sum\limits_{t=1}^{n-1} \frac{X_t^T \tilde{P}_{t+1}X_t}{(1+e^{y_t\tilde{\theta}_t^TX_t})^2}\,.
\end{equation*}
Finally, we  easily control the last two terms of \eqref{eq:regretinterm} as we identify
\begin{equation*}
    \sum\limits_{t=1}^{n-1} \left(\frac{\partial l_t(y_t,\theta)}{\partial \theta}\Bigr|_{\substack{\tilde{\theta}_t}}\right)^T(\tilde{\theta}_t-\tilde{\theta}_{t+1}) = \sum\limits_{t=1}^{n-1} \frac{X_t^T \tilde{P}_{t+1}X_t}{(1+e^{y_t\tilde{\theta}_t^TX_t})^2}\,,
\end{equation*}
and we use the upper-bound 
\begin{equation*}
    \left\|\left(\frac{\partial l_n(y_n,\theta)}{\partial \theta}\Bigr|_{\substack{\tilde{\theta}_n}}\right)^T(\tilde{\theta}_n-\theta)\right\| \le    \left\| \frac{\partial l_n(y_n,\theta)}{\partial \theta}\Bigr|_{\substack{\tilde{\theta}_n}}\right\|(\|\tilde{\theta}_n\|+\|\theta\|)\le  D_X(D_{\theta}+\|\theta\|)\,.
\end{equation*}
Therefore,
\begin{multline*}
    \sum\limits_{t=1}^n \left(l_t(y_t,\tilde{\theta}_t) - l_t(y_t,\theta)\right) \le \left(\frac{\sqrt{d}D_X(D_{\theta}+\|\theta\|)\left(1+e^D\right)}{4}+1\right)\sum\limits_{t=1}^{n-1} \frac{X_t^T \tilde{P}_{t+1}X_t}{(1+e^{y_t\tilde{\theta}_t^TX_t})^2}\\ +\frac12\tilde{\theta}^T_1\tilde{P}_1^{-1}\tilde{\theta}_1+\frac12\theta^T\tilde{P}_1^{-1}\theta+ D_X(D_{\theta}+\|\theta\|)\,.
\end{multline*}
In order to conclude, we follow ideas from \cite{cesa2006prediction} (in particular Lemma 11.11) to prove in Appendix \ref{App1} the following proposition which yields the result of Theorem \ref{regretsemionline}.
\begin{proposition}
\label{cesabianchimodified}
For any sequence $(X_t,y_t)_{1\le t\le n}$ we have 
\begin{equation*}
    \sum\limits_{t=1}^{n-1} \frac{X_t^T\tilde{P}_{t+1}X_t}{(1+e^{y_t\hat{\theta}_t^TX_t})^2} \le \frac{1+e^D}{2}d\log(1+(n-1) p_1D_X^2)\,.
\end{equation*}
\end{proposition}
\end{proof}

\section{Extended Kalman Filter}
\label{EKF}
We were not able to bound the regret of the EKF algorithm in the adversarial setting as we did not control the difference between the matrices $\tilde{P}_t$ and $P_t$. Thus, our EKF regret analysis holds in a restrictive well-specified stochastic setting.
\subsection{Discussion on the assumptions}
 We assume that the stochastic sequence $(X_t,y_t)$ follows the logistic regression model: there exists $\theta_{\text{true}}\in \R^d$ such that
\begin{equation}\label{logmodel}
    p(y_t|X_t,\theta_{\text{true}}) = \frac{1}{1+e^{-y_t\theta_{\text{true}}^TX_t}},\qquad t=1,2,\ldots.
\end{equation}
We do not make any assumption on the dependence of the stochastic process $(X_t)$ so far.
We consider the regret in term of the expected loss conditionally on $X_t$: for any random variable $Z$, we note $\mathbb{E}_t\left[Z\right]$ the conditional expectation $\mathbb{E}[Z\mid X_1,y_1,...,X_{t-1},y_{t-1},X_t]$ (we know the past pairs $(X_s,y_s)_{s<t}$ along with $X_t$ the explanatory variables at time $t$). We first observe that for any $t$, $\theta\rightarrow \mathbb{E}_t[l_t(y_t,\theta)]$ is a convex function minimized in $\theta_{\text{true}}$. Even if $(\mathbb{E}_t[l_t(y_t,\theta)])$ is a stochastic sequence, we apply a convexity argument on the expected losses in order to bound the regret  by a linear regret
\begin{equation*}
    \sum\limits_{t=1}^n \left(\mathbb{E}_t[l_t(y_t,\hat{\theta}_t)] - \mathbb{E}_t[l_t(y_t,\theta_{\text{true}})]\right)
    \le \sum\limits_{t=1}^n \mathbb{E}_t\left[\frac{y_tX_t^T}{1+e^{y_t\hat{\theta}_t^TX_t}}\right](\theta_{\text{true}}-\hat{\theta}_t)\,.
\end{equation*}
All the regret bounds on EKF  provided hereafter actually come from identical bounds on the linear regret.
We identify the expected gradients in the linear regret as  $\mathbb{E}_t [y_tX_t^T/(1+e^{y_t\theta^TX_t})]=\mathbb{E}_{y\sim p(y|X_t,\theta_{\text{true}})} [yX_t^T/(1+e^{y\theta^TX_t})]$. We observe a key property satisfied by the logistic gradients, proved in Appendix \ref{App2}.
\begin{proposition}
\label{propertyexpectedbound}
For any $\theta,X\in\R^d$, there exists $c>0$ satisfying $e^{-|(\theta_{\text{true}}-\theta)^TX|} < c < e^{|(\theta_{\text{true}}-\theta)^TX|}$ and
\begin{equation*}
    \mathbb{E}_{y\sim p(y|X,\theta_{\text{true}})}\left[\frac{yX^T(\theta_{\text{true}}-\theta)}{1+e^{y\theta^TX}}\right] = c \frac{(\theta_{\text{true}}-\theta)^TXX^T(\theta_{\text{true}}-\theta)}{(1+e^{\theta^TX})(1+e^{-\theta^TX})}\,.
\end{equation*}
\end{proposition}
Such  Bernstein's type conditions yield fast rates of convergence. However, the constant $c$ in Proposition \ref{propertyexpectedbound} is relative to the error $|(\theta_{\text{true}}-\theta)^TX|$ and the fast rate holds only locally: If there exists some $\tau$ so that  $|(\theta_{\text{true}}-\hat \theta_t)^TX_t|\le 1/2$ for any $t> \tau$ then an application of Corollary \ref{coroproperty} and Theorem \ref{resulthighproba} (with $\varepsilon=0.5$ and   $\alpha=0.05$ so that $1/2+\alpha=c<e^{-\varepsilon}$) yields the following regret bound
\begin{align*}
    \sum\limits_{t=\tau+1 }^{n+\tau } \mathbb{E}_t\left[\frac{y_tX_t^T}{1+e^{y_t\hat{\theta}_t^TX_t}}\right](\theta_{\text{true}}-\hat{\theta}_t)  \le 30& \Big(20(1+e^D)\log(\frac1\delta) + \frac{1+e^D}{4}d\log(1+n p_1D_X^2 )\\
    &+\frac{1}{2p_1}\|\theta_{\text{true}}\|^2\Big)\,,
\end{align*}
with probability at least $1-\delta$,   $\delta>0$.

In order to get the global regret bound, we need  two extra assumptions on the law of $X_t$:
\begin{assumption}
\label{ass1}
There exists $m_1>0$ such that for any $t=1,2,\ldots$,
\begin{equation*}
    \frac{m_1 I}{t} \prec \mathbb{E}\left[P_{t+1}X_tX_t^T \mid X_1,y_1,\ldots,X_{t-1},y_{t-1}\right]\,.
\end{equation*}
\end{assumption}
\begin{assumption}
\label{ass2}
There exists $M_2>0$ such that for any $t=1,2,\ldots$,
\begin{equation*}
    \mathbb{E}\left[X_t^TP_{t+1}^2 X_t \right] \le \frac{M_2 }{t^2}\,.
\end{equation*}
\end{assumption}
One checks these assumptions under the invertibility of the matrix $\mathbb{E}[XX^T]$ for bounded iid $(X_t)$:
\begin{proposition}\label{prop:iid}
In the iid case,  if  $\lambda_{\rm min}=\lambda_{\rm min}(\mathbb{E}[XX^T])>0$ and if $\|X\|\le D_X$ a.s. then  we have
\begin{align*}
\frac{\lambda_{\rm min}}{t(1+D_X^2)^2}& \le \lambda_{\rm min}\left(\mathbb{E}\left[P_{t+1}X_tX_t^T\mid X_1,y_1,\ldots,X_{t-1},y_{t-1}\right]\right)\,,\\
\lambda_{\rm max}\left(\mathbb{E}\left[P_{t+1}^2\right]\right)&\le \frac{16(1+e^D)^2}{\lambda_{\rm min}^2t^2}\left(1+\frac{1}{t}\frac{2de^{-3}\left(3D_X^4+D_X^2\lambda_{\rm min}/2\right)^3}{(\lambda_{\rm min}^2/8)^2}\right)\,.
\end{align*}
\end{proposition}
The results of Proposition \ref{prop:iid} imply Assumption \ref{ass1} and Assumption \ref{ass2}. Proposition \ref{prop:iid} is proved in Appendix \ref{App2}.

\subsection{Regret bound in expectation for the EKF}
In what follows we assume that 
\begin{equation*}
    D_X\ge \max\limits_{1\le t\le n} \|X_t\|,\ D_{\theta}\ge \max\left(\max\limits_{1\le t\le n} \|\hat{\theta}_t\|,\|\theta_{\text{true}}\|\right)\text{ and } D\ge \max\limits_{1\le t\le n} |\hat{\theta}_t^TX_t|\quad a.s.
\end{equation*}
It is important to note that these constants are not used in the EKF Algorithm \ref{alg:ekf}, making it parameter-free.
\begin{theorem}
\label{resultKalman}
Assume that $(X_t,y_t)$ satisfies the logistic regression \eqref{logmodel} for any $\theta_{\rm true} \in \R^d$. If the EKF starts with $P_1 = p_1 I$, $p_1>0$, $\hat{\theta}_1=0$ and if the assumptions \ref{ass1} and \ref{ass2} are satisfied, we have 
\begin{multline*}
    \sum\limits_{t=1}^n (\mathbb{E}[l(y,\hat{\theta}_t)] - \mathbb{E}[l(y,\theta_{\text{true}})])
    \le 30\Big(20(1+e^D) + \frac{1+e^D}{4}d\log(1+n p_1D_X^2 )+\frac{1}{2p_1}\|\theta_{\text{true}}\|^2\Big) \\
    + (62D_XD_{\theta} +  60D_X^2D_{\theta}^2 + 15D_X^2)\Big(1+ \frac{4^{k+1}D_X^{2k}b_k}{ka-1}\log(n)\Big),\qquad n\ge 1\,.
\end{multline*}
Here $a=e^{-D}m_1/(1+e^D)<1$, $b_k = \frac{5 M_2}{ p_1^2D_X^2}\left(4D_{\theta}^2 + 2p_1D_{\theta}D_X + p_1^2D_X^2\right)^k$ and $k$ is any positive integer satisfying $1<ka<2$.
\end{theorem}
Note that $m_1$ as defined in Assumption \ref{ass1} may be chosen such that $a<1$.

\section{The sketch of the proof of Theorem \ref{resultKalman}}
\label{proofEKF}
One has to distinguish between the localized steps where $|(\theta_{\text{true}}-\hat{\theta}_t)^TX_t| < \varepsilon$ and the others. To this end, we define $T_{\varepsilon} = \{1\le t\le n\ |\ |(\hat{\theta}_t-\theta_{\text{true}})^TX_t| \le \varepsilon\}$. In Section \ref{boundlocal} we exhibit a bound on the sum of the localized terms $t\in T_\varepsilon$, and in Section \ref{boundglobal} we upper-bound the expected value of the non-localized terms $t\notin T_\varepsilon$. Finally Section \ref{merge} merges these results to prove Theorem \ref{resultKalman}. The proofs of the intermediate results used in Sections  \ref{boundlocal}, \ref{boundglobal} and \ref{merge} are deferred to Appendix \ref{App2}.

\subsection{Bounding the localized steps with high probability}
\label{boundlocal}
In the well-specified logistic regression, Proposition \ref{propertyexpectedbound} provides an upper-bound on the linearized regret
$$
\mathbb{E}_t\left[\frac{y_tX_t^T}{1+e^{y_t\hat{\theta}_t^TX_t}}\right](\theta_{\text{true}}-\hat{\theta}_t)=\mathbb{E}_{y\sim p(y|X_t,\theta_{\text{true}})}\left[\frac{yX_t^T}{1+e^{y\hat{\theta}_t^TX_t}}\right](\theta_{\text{true}}-\theta_t)\,.
$$
The following corollary of Proposition \ref{propertyexpectedbound} provides a simple estimate  for localized steps:
\begin{corollary}
\label{coroproperty}
For any step $t=1,2,\ldots$, if we have $|(\theta_{\text{true}}-\hat{\theta}_t)^TX_t| < \varepsilon$ and $c<e^{-\varepsilon}$ then it holds
\begin{align*}
    \mathbb{E}_t\left[\frac{y_tX_t^T}{1+e^{y_t\hat{\theta}_t^TX_t}}\right](\theta_{\text{true}}-\hat{\theta}_t)
    < \frac{e^{\varepsilon}}{e^{-\varepsilon}-c}& \left(\mathbb{E}_t\left[\frac{y_tX_t^T}{1+e^{y_t\hat{\theta}_t^TX_t}}\right](\theta_{\text{true}}-\hat{\theta}_t)\right.\\
    & \left.- c \frac{(\theta_{\text{true}}-\hat{\theta}_t)^TX_tX_t^T(\theta_{\text{true}}-\hat{\theta}_t)}{(1+e^{\hat{\theta}_t^TX_t})(1+e^{-\hat{\theta}_t^TX_t})}\right).
\end{align*}
\end{corollary}
The upper-bound is controlled thanks to the negative quadratic term which is responsible of the fast rate of convergence. It can be seen as a local strong convexity term. We derive the following regret bound with high probability, parameterizing $c=1/2+\alpha$ for some $\alpha>0$:
\begin{theorem}
\label{resulthighproba}
For any $\varepsilon,\alpha,\delta>0$ and starting the EKF from $P_1 = p_1 I$, $p_1>0$ and $\hat\theta_1=0$, we have
\begin{align*}
    \sum\limits_{t\in T_{\varepsilon}} \Big(\mathbb{E}_t\left[\frac{y_tX_t^T}{1+e^{y_t\hat{\theta}_t^TX_t}}\right](\theta_{\text{true}}-\hat{\theta}_t) & - \big(\frac{1}{2}+\alpha\big)\frac{(\theta_{\text{true}}-\hat{\theta}_t)^TX_tX_t^T(\theta_{\text{true}}-\hat{\theta}_t)}{(1+e^{\hat{\theta}_t^TX_t})(1+e^{-\hat{\theta}_t^TX_t})}\Big) \\
    \le & \frac{1+e^D}{\alpha}\log(\delta^{-1})+ \frac{1+e^D}{4}d\log(1+n p_1D_X^2 )+\frac{1}{2p_1}\|\theta_{\text{true}}\|^2 \\
    & - \frac{1}{2} \sum\limits_{t\notin T_{\varepsilon}}
    \left((\hat{\theta}_t-\theta_{\text{true}})^TP_t^{-1}(\hat{\theta}_t-\theta_{\text{true}}) - (\hat{\theta}_{t+1}-\theta_{\text{true}})^TP_{t+1}^{-1}(\hat{\theta}_{t+1}-\theta_{\text{true}})\right),
\end{align*}
with probability at least $1-\delta$.
\end{theorem}

\begin{proof}
We begin with a lemma whose proof is very much inspired by the proof of the regret bound of the ONS algorithm in  \cite{hazan2016introduction}. We note that the constant $1/2$ in front of the quadratic term is responsible to the fast rate of convergence, as there exists a gap with respect to the constant $c\approx 1$ in Proposition  \ref{propertyexpectedbound}.
\begin{lemma}
\label{deterministicbound}
For any $\varepsilon>0$ and any sequence $(X_t,y_t)$, starting the EKF from $P_1 = p_1 I$, $p_1>0$ and $\hat\theta_1=0$, we have 
\begin{align*}
    \sum\limits_{t\in T_{\varepsilon}} \Big(\frac{y_tX_t^T}{1+e^{y_t\hat{\theta}_t^TX_t}}(\theta_{\text{true}}-\hat{\theta}_t) & - \frac{1}{2}(\theta_{\text{true}}-\hat{\theta}_t)^T(P_{t+1}^{-1}-P_t^{-1})(\theta_{\text{true}}-\hat{\theta}_t)\Big) \\
    \le &\frac{1+e^D}{4}d\log(1+n p_1D_X^2 )+\frac{1}{2p_1}\|\theta_{\text{true}}\|^2 \\
    &- \frac{1}{2} \sum\limits_{t\notin T_{\varepsilon}}
    \left((\hat{\theta}_t-\theta_{\text{true}})^TP_t^{-1}(\hat{\theta}_t-\theta_{\text{true}}) - (\hat{\theta}_{t+1}-\theta_{\text{true}})^TP_{t+1}^{-1}(\hat{\theta}_{t+1}-\theta_{\text{true}})\right).
\end{align*}
\end{lemma}
Then we prove the following lemma which is a corollary of a martingale inequality from \cite{bercu2008exponential}:
\begin{lemma}
\label{lemmamartingale}
We define $\Delta M_t=\mathbb{E}_t\left[\frac{y_tX_t^T}{1+e^{y_t\hat{\theta}_t^TX_t}}\right](\theta_{\text{true}}-\hat{\theta}_t) - \frac{y_tX_t^T}{1+e^{y_t\hat{\theta}_t^TX_t}}(\theta_{\text{true}}-\hat{\theta}_t)$. Then for any $\varepsilon>0$ and  $\alpha>0$, it holds
\begin{equation*}
    \label{martingalebound}
    \sum\limits_{t\in T_{\varepsilon}} \left(\Delta M_t - \alpha\frac{(\theta_{\text{true}}-\hat{\theta}_t)^TX_tX_t^T(\theta_{\text{true}}-\hat{\theta}_t)}{(1+e^{\hat{\theta}_t^TX_t})(1+e^{-\hat{\theta}_t^TX_t})}\right) \le \frac{1+e^D}{\alpha}\log(\delta^{-1}),
\end{equation*}
with probability at least $1-\delta$.
\end{lemma}
Adding the inequalities of Lemma \ref{deterministicbound} and Lemma \ref{lemmamartingale} yield the result.
\end{proof}
From Theorem \ref{resulthighproba} it is easy to bound the linearized regret of the localized steps in expectation:
\begin{corollary}
\label{boundexpectationfirst}
For any $\varepsilon,\alpha>0$, we have
\begin{align*}
    \sum\limits_{t\in T_{\varepsilon}}  \Big(\mathbb{E}\left[\frac{y_tX_t^T}{1+e^{y_t\hat{\theta}_t^TX_t}}(\theta_{\text{true}}-\hat{\theta}_t)\right] & - (\frac{1}{2}+\alpha)\frac{(\theta_{\text{true}}-\hat{\theta}_t)^TX_tX_t^T(\theta_{\text{true}}-\hat{\theta}_t)}{(1+e^{\hat{\theta}_t^TX_t})(1+e^{-\hat{\theta}_t^TX_t})}\Big) \\
    \le & \frac{1+e^D}{\alpha}+ \frac{1+e^D}{4}d\log(1+n p_1D_X^2 )+\frac{1}{2p_1}\|\theta_{\text{true}}\|^2 \\
    &- \frac{1}{2} \sum\limits_{t\notin T_{\varepsilon}}
    \left((\hat{\theta}_t-\theta_{\text{true}})^TP_t^{-1}(\hat{\theta}_t-\theta_{\text{true}}) - (\hat{\theta}_{t+1}-\theta_{\text{true}})^TP_{t+1}^{-1}(\hat{\theta}_{t+1}-\theta_{\text{true}})\right).
\end{align*}
\end{corollary}

\subsection{Bounding the expected number of unlocalized steps}
\label{boundglobal}
It is essential in our proof to lower-bound the cardinal of $T_{\varepsilon}$ because the remaining terms in the regret of Theorem \ref{resultKalman} are essentially controlled by it:
\begin{multline}
    \label{boundcardinal}
    \sum\limits_{t\notin T_{\varepsilon}} \Big\|\mathbb{E}_t\Big[\frac{y_tX_t^T}{1+e^{y_t\hat{\theta}_t^TX_t}}(\theta_{\text{true}}-\hat{\theta}_t)\Big]\\
     - \frac{e^{\varepsilon}}{e^{-\varepsilon}-(\frac{1}{2}+\alpha)} \frac{1}{2} \left((\hat{\theta}_t-\theta_{\text{true}})^TP_t^{-1}(\hat{\theta}_t-\theta_{\text{true}}) - (\hat{\theta}_{t+1}-\theta_{\text{true}})^TP_{t+1}^{-1}(\hat{\theta}_{t+1}-\theta_{\text{true}})\right)\Big\| \\
    \le \left(2D_XD_{\theta} + \frac{e^{\varepsilon}}{e^{-\varepsilon}-(\frac{1}{2}+\alpha)}(2D_XD_{\theta} + 2D_X^2D_{\theta}^2 + \frac{1}{2}D_X^2) \right) \left(n-Card(T_{\varepsilon})\right).
\end{multline}
We bound the number of non-localized steps $n-Card(T_{\varepsilon})$ in expectation only, yielding the regret bound in Theorem \ref{resultKalman} in expectation. 
\begin{theorem}
\label{boundexpectationsecond}
We define $a=e^{-D}m_1/(1+e^D)$.
For any positive integer $k$ such that $1<ka<2$, provided that Assumptions \ref{ass1} and \ref{ass2} are satisfied, we have
 \begin{equation*}
    \mathbb{E}\left[n-Card(T_{\varepsilon})\right] \le 1+ \frac{4D_X^{2k}b_k}{\varepsilon^{2k}(ka-1)}\log(n).
\end{equation*}
with $b_k=   \frac{5 M_2}{ p_1^2D_X^2}\left(4D_{\theta}^2 + 2p_1D_{\theta}D_X + p_1^2D_X^2\right)^k $.
\end{theorem}
\begin{proof}
We first find a recursive bound on $\mathbb{E}\left[\|\hat{\theta}_t-\theta_{\text{true}}\|^{2k}\right]$ for any positive integer $k$:
\begin{lemma}
\label{lemmarecursion}
Provided that Assumptions \ref{ass1} and \ref{ass2} are satisfied, it holds
\begin{equation*}
    \mathbb{E}\left[\|\hat{\theta}_{t+1}-\theta_{\text{true}}\|^{2k} \right] \le \mathbb{E}\left[\|\hat{\theta}_t-\theta_{\text{true}}\|^{2k} \right]\Big(1-\frac{ka}{t}\Big) + \frac{b_k}{t^{2}},\qquad t\ge 1.
\end{equation*}
\end{lemma}
It is easy to derive from this lemma the following corollary:
\begin{corollary}
If Assumptions \ref{ass1} and \ref{ass2} are satisfied and $1<ka<2$ then we have
\label{boundpowerk}
\begin{equation*}
    \mathbb{E}\left[\|\hat{\theta}_t-\theta_{\text{true}}\|^{2k}\right] \le \frac{4 b_k}{t(ka-1)},\qquad t\ge 2.
\end{equation*}
\end{corollary}
Then, using first the Cauchy-Schwarz inequality $|(\hat{\theta}_t-\theta_{\text{true}})^TX_t| \le \|\hat{\theta}_t-\theta_{\text{true}}\| \|X_t\|$ and second the Markov inequality, we get
\begin{align*}
    \mathbb{P}\left(|(\hat{\theta}_t-\theta_{\text{true}})^TX_t| > \varepsilon\right) & \le \mathbb{P}\left(\|\hat{\theta}_t-\theta_{\text{true}}\| > \frac{\varepsilon}{\|X_t\|}\right) \\
    & \le \frac{\mathbb{E}\left[\|\hat{\theta}_t-\theta_{\text{true}}\|^{2k}\right]}{\varepsilon^{2k}/\|X_t\|^{2k}} \\
    & \le \frac{4\|X_t\|^{2k}b_k}{\varepsilon^{2k}(ka-1)}\frac{1}{t}\,.
\end{align*}
which proves the Theorem by a summation argument for $2\le t\le n$ together with the trivial bound $\mathbb{P}\left(|(\hat{\theta}_1-\theta_{\text{true}})^TX_1| > \varepsilon\right)\le 1$.
\end{proof}
\subsection{Proof of Theorem \ref{resultKalman}}
\label{merge}
Summing Corollary \ref{boundexpectationfirst} and Equation \ref{boundcardinal} along with Theorem \ref{boundexpectationsecond} yields the regret bound in expectation
\begin{align*}
    \sum\limits_{t=1}^n& \left(\mathbb{E}[l(y,\hat{\theta}_t)] - \mathbb{E}[l(y,\theta_{\text{true}})]\right)\\
    \le &\frac{e^{\varepsilon}}{e^{-\varepsilon}-(\frac{1}{2}+\alpha)}\Big(\frac{1+e^D}{\alpha}+ \frac{1+e^D}{4}d\log(1+n p_1D_X^2 )+\frac{1}{2p_1}\|\theta_{\text{true}}\|^2\Big) \\
    &+ \Big(2D_XD_{\theta} + \frac{e^{\varepsilon}}{e^{-\varepsilon}-(\frac{1}{2}+\alpha)}(2D_XD_{\theta} + 2D_X^2D_{\theta}^2 + \frac{1}{2}D_X^2) \Big)\Big(1+ \Big(\frac{D_X}{\varepsilon}\Big)^{2k}\frac{4b_k}{ka-1}\log(n)\Big)\,,
\end{align*}
for any $0<\alpha<\frac{1}{2}$ and $0<\varepsilon<-\log(1/2+\alpha)$. Choosing $\varepsilon=0.5$ and $\alpha=0.05$, we get $\frac{e^{\varepsilon}}{e^{-\varepsilon}-(\frac{1}{2}+\alpha)}\approx 29.2 < 30$ and we obtain the result of Theorem \ref{resultKalman}.

\section{Conclusion and future work}
\label{future}
We have designed an algorithm called SOS that is a second-order algorithm very close to the EKF. We obtain a compromise between time complexity and regret guarantee. Indeed, its complexity lies between FTL, which is computationally greedy, and the EKF, which is the least expensive second-order algorithm. Moreover, we prove that SOS achieves the optimal $\mathcal{O}(\log(n))$ regret in the adversarial setting. The dependence of the constant of Theorem \ref{regretsemionline} on the exp-concavity constant might be reduced.

An interesting challenge is to adjust this regret bound for the EKF. In this paper we obtained weaker guarantees for the EKF, a $\mathcal{O}(\log(n))$ regret bound with prohibitive constants, in expectation and in the well-specified logistic regression. It would be interesting to obtain the result in the misspecified setting. Another improvement would be to find conditions for the convergence of $\sum_t \mathbb{P}(|(\hat{\theta}_t-\theta_{\text{true}})^TX_t| > \varepsilon)$, transforming our regret in expectation to a regret with high probability.

Also, as though we focused on logistic regression, it seems to us that our approach might be applied to any Generalized Linear Model.


\bibliography{mybib.bib}

\appendix

\section{Details for the proof of Theorem \ref{regretsemionline}}\label{App1}
\begin{proof}[Proof of Lemma \ref{lemmaSincrement}.]
Let $S_t^{(i)}(\theta)=\frac{\partial}{\partial \theta_i} \left[\sum\limits_{s=1}^{t-1} l_s(y_s,\theta) +\frac12\theta^T \tilde{P}_1^{-1} \theta \right]$ be the $i^{th}$ coordinate of $S_t(\theta)$. We prove that
\begin{equation}\label{eq:marg}
    \left|S_{t+1}^{(i)}(\tilde{\theta}_{t+1}) - S_t^{(i)}(\tilde{\theta}_t) \right| \le \frac{D_X\left(1+e^D\right)}{4}\frac{X_t^T \tilde{P}_{t+1}X_t}{(1+e^{y_t\tilde{\theta}_t^TX_t})^2}.
\end{equation}
Applying a Taylor expansion, there exists  $0 < \alpha_t^{(i)} < 1$ satisfying
\begin{equation*}
    S_{t+1}^{(i)}(\tilde{\theta}_{t+1}) - S_{t+1}^{(i)}(\tilde{\theta}_t) - \left(\frac{\partial S_{t+1}^{(i)}(\theta)}{\partial \theta}\Bigr|_{\substack{\tilde{\theta}_t}}\right)^T(\tilde{\theta}_{t+1} - \tilde{\theta}_t) = \frac{1}{2} (\tilde{\theta}_{t+1} - \tilde{\theta}_t)^T \left(\frac{\partial^2 S_{t+1}^{(i)}(\theta)}{\partial \theta^2}\Bigr|_{\substack{\tilde{\theta}_t + \alpha_t^{(i)}(\tilde{\theta}_{t+1}-\tilde{\theta}_t)}}\right) (\tilde{\theta}_{t+1} - \tilde{\theta}_t).
\end{equation*}
Thanks to the update of $\tilde{\theta}_{t+1}$, we have the relation
\begin{equation*}
    S_{t+1}^{(i)}(\tilde{\theta}_t) - S_t^{(i)}(\tilde{\theta}_t) + \left( \frac{\partial S_{t+1}^{(i)}(\theta)}{\partial \theta}\Bigr|_{\substack{\tilde{\theta}_t}}\right)^T(\tilde{\theta}_{t+1} - \tilde{\theta}_t) = 0.
\end{equation*}
Therefore, summing last two identities, we get
\begin{equation*}
    S_{t+1}^{(i)}(\tilde{\theta}_{t+1}) - S_t^{(i)}(\tilde{\theta}_t)    = \frac{1}{2} (\tilde{\theta}_{t+1} - \tilde{\theta}_t)^T \left(\frac{\partial^2 S_{t+1}^{(i)}(\theta)}{\partial \theta^2}\Bigr|_{\substack{\tilde{\theta}_t + \alpha_t^{(i)}(\tilde{\theta}_{t+1}-\tilde{\theta}_t)}}\right) (\tilde{\theta}_{t+1} - \tilde{\theta}_t).
\end{equation*}
From the definition of the logistic loss, we identify
\begin{equation*}
    \frac{\partial^2 S_{t+1}^{(i)}(\theta)}{\partial \theta^2} = -\sum\limits_{s=1}^t \frac{X_s^{(i)}X_sX_s^T(e^{\theta^TX_s}-e^{-\theta^TX_s})}{(1+e^{\theta^TX_s})^2(1+e^{-\theta^TX_s})^2},\qquad \theta\in \R^d.
\end{equation*}
Using the bound $\frac{|e^{\theta^TX_s}-e^{-\theta^TX_s}|}{(1+e^{\theta^TX_s})(1+e^{-\theta^TX_s})} < 1$, we have
\begin{multline*}
    \left|\frac{1}{2} (\tilde{\theta}_{t+1} - \tilde{\theta}_t)^T \left(\frac{\partial^2 S_{t+1}^{(i)}(\theta)}{\partial \theta^2}\Bigr|_{\substack{\tilde{\theta}_t + \alpha(\tilde{\theta}_{t+1}-\tilde{\theta}_t)}}\right) (\tilde{\theta}_{t+1} - \tilde{\theta}_t) \right|
    \le  \frac{D_X}{2}(\tilde{\theta}_{t+1} - \tilde{\theta}_t)^T \\
    \left(\sum\limits_{s=1}^t \frac{X_sX_s^T}{(1+e^{(\tilde{\theta}_t + \alpha^{(i)}(\tilde{\theta}_{t+1}-\tilde{\theta}_t))^TX_s})(1+e^{-(\tilde{\theta}_t + \alpha^{(i)}(\tilde{\theta}_{t+1}-\tilde{\theta}_t))^TX_s})}\right)(\tilde{\theta}_{t+1} - \tilde{\theta}_t).
\end{multline*}
Noticing that
\begin{equation*}
    \frac{1}{(1+e^{(\tilde{\theta}_t + \alpha^{(i)}(\tilde{\theta}_{t+1}-\tilde{\theta}_t))^TX_s})(1+e^{-(\tilde{\theta}_t + \alpha^{(i)}(\tilde{\theta}_{t+1}-\tilde{\theta}_t))^TX_s})} < \frac{1+e^D}{2} \frac{1}{(1+e^{\tilde{\theta}_t^TX_s})(1+e^{-\tilde{\theta}_t^TX_s})},
\end{equation*}
we obtain
\begin{equation*}
    \left(\sum\limits_{s=1}^t \frac{X_sX_s^T}{(1+e^{(\tilde{\theta}_t + \alpha^{(i)}(\tilde{\theta}_{t+1}-\tilde{\theta}_t))^TX_s})(1+e^{-(\tilde{\theta}_t + \alpha^{(i)}(\tilde{\theta}_{t+1}-\tilde{\theta}_t))^TX_s})}\right) \prec \frac{1+e^D}{2} \tilde{P}_{t+1}^{-1}.
\end{equation*}
Combining our findings with the updates $\tilde{\theta}_{t+1} - \tilde{\theta}_t=\tilde{P}_{t+1}\frac{y_tX_t}{1+e^{y_t\tilde{\theta}_t^T X_t}}$ yields Eqn. \eqref{eq:marg}. The desired result follows easily as 
$$
 \left\|S_{t+1} (\tilde{\theta}_{t+1}) - S_t (\tilde{\theta}_t) \right\|\le \sqrt d\max_{1\le i\le d} \left|S_{t+1}^{(i)}(\tilde{\theta}_{t+1}) - S_t^{(i)}(\tilde{\theta}_t) \right|\,.
 $$
\end{proof}

\begin{proof}[Proof of Proposition \ref{cesabianchimodified}.]
For any $1\le s\le t\le n$ we have $\frac{1}{(1+e^{\tilde{\theta}_t^TX_s})(1+e^{-\tilde{\theta}_t^TX_s})} > \frac{1}{2(1+e^D)}$  and 
\begin{equation*}
    \frac{X_sX_s^T}{(1+e^{\tilde{\theta}_t^TX_s})(1+e^{-\tilde{\theta}_t^TX_s})} \succ \frac{X_sX_s^T}{2(1+e^D)}.
\end{equation*}
Summing this inequality  along with $1>\frac{1}{2(1+e^D)}$ and $\tilde{P}_1^{-1}\succ 0$ yields
\begin{equation*}
    \tilde{P}_{t+1}^{-1}= \tilde{P}_1^{-1} + \sum\limits_{s=1}^t \frac{X_sX_s^T}{(1+e^{\tilde{\theta}_t^TX_s})(1+e^{-\tilde{\theta}_t^TX_s})} \succ \frac{1}{2(1+e^D)}\left(\tilde{P}_1^{-1}+\sum\limits_{s=1}^t X_sX_s^T\right).
\end{equation*}
Using that if $A$ and $B$ are positive definite matrices, $A\succ B \implies A^{-1}\prec B^{-1}$, we get
\begin{equation*}
    \tilde{P}_{t+1} \prec 2(1+e^D)\left(\tilde{P}_1^{-1}+\sum\limits_{s=1}^t X_sX_s^T\right)^{-1}.
\end{equation*}
We then apply Lemma 11.11 of \cite{cesa2006prediction}   to get
\begin{align*}
    X_t^T\left(\tilde{P}_1^{-1}+\sum\limits_{s=1}^t X_sX_s^T\right)^{-1}X_t &= 1 - \frac{\det\left(\tilde{P}_1^{-1}+\sum\limits_{s=1}^{t-1} X_sX_s^T\right)}{\det\left(\tilde{P}_1^{-1}+\sum\limits_{s=1}^t X_sX_s^T\right)} \\
    &\le \log\left(\frac{\det\left(\tilde{P}_1^{-1}+\sum\limits_{s=1}^t X_sX_s^T\right)}{\det\left(\tilde{P}_1^{-1}+\sum\limits_{s=1}^{t-1} X_sX_s^T\right)}\right),
\end{align*}
thanks to the inequality $1-x\le \log(1/x)$ for any $x>0$. As $\frac{1}{(1+e^{\tilde{\theta}_t^TX_t})(1+e^{-\tilde{\theta}_t^TX_t})} \le \frac{1}{4}$, we get
\begin{equation*}
    \frac{X_t^T\tilde{P}_{t+1}X_t}{(1+e^{\tilde{\theta}_t^TX_t})(1+e^{-\tilde{\theta}_t^TX_t})} \le \frac{1+e^D}{2} \log\left(\frac{\det\left(\tilde{P}_1^{-1}+\sum\limits_{s=1}^t X_sX_s^T\right)}{\det\left(\tilde{P}_1^{-1}+\sum\limits_{s=1}^{t-1} X_sX_s^T\right)}\right).
\end{equation*}
Summing from $1$ to $n-1$, we obtain
\begin{align*}
    \sum\limits_{t=1}^{n-1}\frac{X_t^T\tilde{P}_{t+1}X_t}{(1+e^{\tilde{\theta}_t^TX_t})(1+e^{-\tilde{\theta}_t^TX_t})} & \le \frac{1+e^D}{2} \log\left(\frac{\det\left(\tilde{P}_1^{-1}+\sum\limits_{s=1}^{n-1} X_sX_s^T\right)}{\det(\tilde{P}_1^{-1})}\right) \\
    & = \frac{1+e^D}{2} \log\left(\det\left(I+p_1\sum\limits_{s=1}^{n-1} X_sX_s^T\right)\right),
\end{align*}
because $\tilde{P}_1=p_1 I$. The biggest eigenvalue of $\left(I+p_1\sum\limits_{s=1}^{n-1} X_sX_s^T\right)$ is bounded by $1+p_1\sum\limits_{s=1}^{n-1} X_s^TX_s$. We get
\begin{equation*}
    \det\left(I+p_1\sum\limits_{s=1}^{n-1} X_sX_s^T\right) \le \left(1+p_1\sum\limits_{s=1}^{n-1} X_s^TX_s\right)^d \le \left(1+p_1(n-1)D_X^2\right)^d.
\end{equation*}
and that concludes the proof.
\end{proof}

\section{Details for the proof of Theorem \ref{resultKalman}}\label{App2}
\begin{proof}[Proof of Proposition \ref{propertyexpectedbound}.]
We develop the expectation
\begin{equation*}
    \mathbb{E}_{y\sim p(y|X,\theta_{\text{true}})}\left[\frac{yX^T(\theta_{\text{true}}-\theta)}{1+e^{y\theta^TX}}\right] = \frac{(\theta_{\text{true}}-\theta)^TX}{(1+e^{\theta^TX})(1+e^{-\theta^TX})} \left[\frac{1+e^{-\theta^TX}}{1+e^{-\theta_{\text{true}}^TX}} - \frac{1+e^{\theta^TX}}{1+e^{\theta_{\text{true}}^TX}}\right].
\end{equation*}
Therefore, we bound $\left[\frac{1+e^{-\theta^TX}}{1+e^{-\theta_{\text{true}}^TX}} - \frac{1+e^{\theta^TX}}{1+e^{\theta_{\text{true}}^TX}}\right]$ in terms of $X^T(\theta_{\text{true}}-\theta)$. We first rearrange it:
\begin{align*}
    \frac{1+e^{-\theta^TX}}{1+e^{-\theta_{\text{true}}^TX}} - \frac{1+e^{\theta^TX}}{1+e^{\theta_{\text{true}}^TX}} & = \left(1+\frac{e^{-\theta^TX}-e^{-\theta_{\text{true}}^TX}}{1+e^{-\theta_{\text{true}}^TX}}\right) - \left(1+\frac{e^{\theta^TX}-e^{\theta_{\text{true}}^TX}}{1+e^{\theta_{\text{true}}^TX}}\right) \\
    & = \frac{e^{(\theta_{\text{true}}-\theta)^TX}-1}{1+e^{\theta_{\text{true}}^TX}} - \frac{e^{-(\theta_{\text{true}}-\theta)^TX}-1}{1+e^{-\theta_{\text{true}}^TX}}.
\end{align*}
There exists $|\alpha|\le |(\theta_{\text{true}}-\theta)^TX|$ such that
\begin{equation*}
    \frac{e^{(\theta_{\text{true}}-\theta)^TX}-1}{1+e^{\theta_{\text{true}}^TX}} - \frac{e^{-(\theta_{\text{true}}-\theta)^TX}-1}{1+e^{-\theta_{\text{true}}^TX}} = \left[\frac{e^{\alpha}}{1+e^{\theta_{\text{true}}^TX}} + \frac{e^{-\alpha}}{1+e^{-\theta_{\text{true}}^TX}}\right] X^T(\theta_{\text{true}}-\theta).
\end{equation*}
As the function $x\rightarrow \frac{a}{1+x} + \frac{b}{1+1/x}$ is monotonic with limits $a$ and $b$ in $0$ and $+\infty$ respectively, we get 
$$
e^{-(\theta_{\text{true}}-\theta)^TX} < \frac{e^{\alpha}}{1+e^{\theta_{\text{true}}^TX}} + \frac{e^{-\alpha}}{1+e^{-\theta_{\text{true}}^TX}} < e^{(\theta_{\text{true}}-\theta)^TX}.
$$
\end{proof}

\begin{proof}[Proof of Proposition \ref{prop:iid}.]
We define $S_{t+1}=\left(I+\sum\limits_{s=1}^t X_sX_s^T\right)^{-1}$. As $S_t\prec P_t\prec 2(1+e^D)S_t$, we give the desired results first on $S_t$ and the desired results on $P_t$ follow easily.

We first give a lower bound on $\mathbb{E}[S_{t+1}X_tX_t^T\mid X_1,y_1,\ldots,X_{t-1},y_{t-1}]$. Using the relation $S_{t+1}=S_t-\frac{S_tX_tX_t^TS_t}{1+X_t^TS_tX_t}$, we write
\begin{equation*}
    S_{t+1}X_tX_t^T = S_tX_tX_t^T-\frac{S_tX_tX_t^TS_t}{1+X_t^TS_tX_t}X_tX_t^T = \frac{1}{1+X_t^TS_tX_t}S_tX_tX_t^T.
\end{equation*}
Noting that $S_t\succcurlyeq I$, $\|X_t\|\le D_X$, and $S_tX_tX_t^T\succcurlyeq 0$ as a rank-one matrix with eigenvalue $X_t^TS_tX_t>0$, we get $S_{t+1}X_tX_t^T\succcurlyeq \frac{S_tX_tX_t^T}{1+D_X^2}$. It implies
\begin{align*}
    \mathbb{E}[S_{t+1}X_tX_t^T\mid X_1,y_1,\ldots,X_{t-1},y_{t-1}] & \succcurlyeq \frac{1}{1+D_X^2} \mathbb{E}[S_tX_tX_t^T\mid X_1,y_1,\ldots,X_{t-1},y_{t-1}]\\
    & = \frac{1}{1+D_X^2} S_t \mathbb{E}[X_tX_t^T\mid X_1,y_1,\ldots,X_{t-1},y_{t-1}]\,.
\end{align*}
The independence hypothesis yields $\mathbb{E}[X_tX_t^T\mid X_1,y_1,\ldots,X_{t-1},y_{t-1}]=\mathbb{E}[XX^T]$. Also, from $\lambda_{\rm max}(S_t^{-1})\le 1+(t-1)D_X^2$ we obtain $\lambda_{\rm min}(S_t)\ge \frac{1}{1+(t-1)D_X^2}\ge \frac{1}{t(1+D_X^2)}$. Therefore
\begin{equation*}
    \lambda_{\rm min}(\mathbb{E}[S_{t+1}X_tX_t^T\mid X_1,y_1,\ldots,X_{t-1},y_{t-1}]) \ge \frac{1}{1+D_X^2} \lambda_{\rm min}(S_t) \lambda_{\rm min}(\mathbb{E}[XX^T])
    \ge \frac{\lambda_{\rm min}}{t(1+D_X^2)^2}\,.
\end{equation*}
In order to get an upper bound on $\lambda_{\rm max}(\mathbb{E}[P_{t+1}^2])$, we first bound $\mathbb{P}\left(\lambda_{\rm max}(S_t)>\frac{2}{t\lambda_{\rm min}}\right)$. Then we estimate $\lambda_{\rm max}(\mathbb{E}[P_{t+1}^2])\le \mathbb{E}[\lambda_{\rm max}(P_{t+1}^2)]\le 4(1+e^D)^2\mathbb{E}[\lambda_{\rm max}(S_t)^2]$ with
\begin{align*}
    \mathbb{E}[\lambda_{\rm max}(S_t)^2] & = \mathbb{E}[\lambda_{\rm max}(S_t)^2\mathds{1}_{\lambda_{\rm max}(S_t)>\frac{2}{t\lambda_{\rm min}}}] + \mathbb{E}[\lambda_{\rm max}(S_t)^2\mathds{1}_{\lambda_{\rm max}(S_t)\le\frac{2}{t\lambda_{\rm min}}}] \\
    & \le \mathbb{P}\left(\lambda_{\rm max}(S_t)>\frac{2}{t\lambda_{\rm min}}\right) + \frac{4}{t^2\lambda_{\rm min}^2}\,,
\end{align*}
because $\lambda_{\rm max}(S_t)\le 1$ and $\mathbb{P}\left(\lambda_{\rm max}(S_t)\le\frac{2}{t\lambda_{\rm min}}\right)\le1$.

We control the deviations of $\lambda_{\rm max}(S_t)$ first by centering as
\begin{align*}
    \mathbb{P}\left(\lambda_{\rm max}(S_t)>\frac{2}{t\lambda_{\rm min}}\right) & = \mathbb{P}\left(\lambda_{\rm min}(S_t^{-1})<\frac{\lambda_{\rm min}}{2} t\right) \\
    & = \mathbb{P}\left(\lambda_{\rm min}\Big(\sum\limits_{s=1}^tX_sX_s^T\Big)<\frac{\lambda_{\rm min}}{2} t-1\right) \\
    & \le \mathbb{P}\left(\lambda_{\rm min}\Big(\sum\limits_{s=1}^tX_sX_s^T\Big)<\frac{\lambda_{\rm min}}{2} t\right) \\
    & = \mathbb{P}\left(\lambda_{\rm min}\Big(\sum\limits_{s=1}^tX_sX_s^T\Big)-t\lambda_{\rm min}(\mathbb{E}[XX^T])<-\frac{\lambda_{\rm min}}{2} t\right)\,.
\end{align*}
Then we want to show that  $X_sX_s^T$ and $\mathbb{E}[XX^T]$ are commuting in order to rewrite the centered smallest eigenvalue as the smallest eigenvalue of a centered matrix and apply the Bernstein inequality of \cite{tropp2011user} on it. We note that 
$X_sX_s^T\mathbb{E}[XX^T]$ is a rank-one matrix with eigenvalue $X_s^T\mathbb{E}[XX^T]X_s>0$ if $X_s\neq 0$. Therefore $X_sX_s^T\mathbb{E}[XX^T]\succcurlyeq 0$ is symmetric and $X_sX_s^T \mathbb{E}[XX^T]= (X_sX_s^T \mathbb{E}[XX^T])^T= \mathbb{E}[XX^T]X_sX_s^T$. Similarly, we get that $\sum_1^tX_sX_s^T$ and $\mathbb{E}[XX^T]$ are commuting, thus they are simultaneously diagonalizable. Using their joint diagonalization, we infer that
$$
\lambda_{\rm min}\left(\sum\limits_{s=1}^t(X_sX_s^T-\mathbb{E}[XX^T])\right)\le \lambda_{\rm min}\left(\sum\limits_{s=1}^tX_sX_s^T\right)-\lambda_{\rm min}\left(\sum\limits_{s=1}^t\mathbb{E}[XX^T]\right).
$$
Combining those results, we obtain that
\begin{align*}
    \mathbb{P}\left(\lambda_{\rm max}(S_t)>\frac{2}{t\lambda_{\rm min}}\right) & \le \mathbb{P}\left(\lambda_{\rm min}\left(\sum\limits_{s=1}^t(X_sX_s^T-\mathbb{E}[XX^T])\right)<-\frac{\lambda_{\rm min}}{2} t\right) \\
    & = \mathbb{P}\left(\lambda_{\rm max}\left(\sum\limits_{s=1}^t(\mathbb{E}[XX^T]-X_sX_s^T)\right)>\frac{\lambda_{\rm min}}{2} t\right).
\end{align*}
We apply Theorem 1.3 of \cite{tropp2011user} which is a Bernstein inequality on the largest eigenvalue of sums of independent centered matrices. We check the conditions:
\begin{itemize}
    \item $\mathbb{E}\left[\mathbb{E}[XX^T]-X_sX_s^T\right]=0$,
    \item $\lambda_{\rm max}(\mathbb{E}[XX^T]-X_sX_s^T)\le D_X^2$ a.s.,
    \item from $\|X_s\|\le D_X^2$ as $0\preccurlyeq \mathbb{E}\left[(\mathbb{E}[XX^T]-X_sX_s^T)^2\right] \preccurlyeq \mathbb{E}\left[(X_sX_s^T)^2\right]$, the largest singular value of $\mathbb{E}\left[(\mathbb{E}[XX^T]-X_sX_s^T)^2\right]$ is upper-bounded by $D_X^4$.
\end{itemize}
Therefore we obtain
$$
\mathbb{P}\left(\lambda_{\rm max}\left(\sum\limits_{s=1}^t(\mathbb{E}[XX^T]-X_sX_s^T)\right)>u \right)\le d\exp\left(-\frac{u^2/2}{tD_X^4+D_X^2u/3}\right),\qquad u>0.
$$
Applying it with $u=\frac{\lambda_{\rm min}}{2}t$, we get
\begin{align*}
    \mathbb{P}\left(\lambda_{\rm max}(S_t)>\frac{2}{t\lambda_{\rm min}}\right) & \le d \exp\left(-\frac{(\lambda_{\rm min}/2)^2t^2/2}{tD_X^4+D_X^2(\lambda_{\rm min}/2)t/3}\right)\\
    & = d \exp\left(-t\frac{\lambda_{\rm min}^2/8}{D_X^4+D_X^2\lambda_{\rm min}/6}\right)\\
    & \le \frac{1}{t^3}\frac{27de^{-3}\left(D_X^4+D_X^2\lambda_{\rm min}/6\right)^3}{(\lambda_{\rm min}^2/8)^3},
\end{align*}
because $\max\limits_{x\in\R}(e^{-ax}x^3)=\frac{27e^{-3}}{a^3}$ for any $a>0$. Therefore the desired result is obtained as
\begin{equation*}
    \mathbb{E}[\lambda_{\rm max}(S_t)^2] \le \frac{4}{\lambda_{\rm min}^2 t^2}\left(1+\frac{1}{t}\frac{2de^{-3}\left(3D_X^4+D_X^2\lambda_{\rm min}/2\right)^3}{(\lambda_{\rm min}^2/8)^2}\right)\,.
\end{equation*}
\end{proof}

\begin{proof}[Proof of Corollary \ref{coroproperty}.]
Denoting $E_t=\mathbb{E}_t\left[\frac{y_tX_t^T}{1+e^{y_t\hat{\theta}_t^TX_t}}\right](\theta_{\text{true}}-\hat{\theta}_t)$ and $Q_t=\frac{(\theta_{\text{true}}-\hat{\theta}_t)^TX_tX_t^T(\theta_{\text{true}}-\hat{\theta}_t)}{(1+e^{\hat{\theta}_t^TX_t})(1+e^{-\hat{\theta}_t^TX_t})}$, we have the sandwich relationship $e^{-\varepsilon}Q_t < E_t < e^{\varepsilon}Q_t$ according to Proposition \ref{propertyexpectedbound}. Therefore we obtain
\begin{equation*}
    E_t-cQ_t > (e^{-\varepsilon}-c)Q_t > \frac{e^{-\varepsilon}-c}{e^{\varepsilon}}E_t,
\end{equation*}
and the Corollary follows.
\end{proof}

\begin{proof}[Proof of Lemma \ref{deterministicbound}.]
We start from the Kalman recursion
\begin{align}
    P_{t+1}^{-1} &= P_t^{-1} + \frac{1}{(1+e^{\hat{\theta}_t^TX_t})(1+e^{-\hat{\theta}_t^TX_t})}X_tX_t^T,\nonumber\\
    \label{updatethetalemma}
    \hat{\theta}_{t+1} &= \hat{\theta}_t + P_{t+1}\frac{y_tX_t}{1+e^{y_t\hat{\theta}_t^TX_t}}.
\end{align}
Multiplying Equation \eqref{updatethetalemma} by $P_{t+1}^{-1}$ and \eqref{updatethetalemma} and subtracting $\theta_{\text{true}}$, we obtain
\begin{equation*}
    (\hat{\theta}_{t+1}-\theta_{\text{true}})^TP_{t+1}^{-1}(\hat{\theta}_{t+1}-\theta_{\text{true}}) = (\hat{\theta}_t-\theta_{\text{true}})^TP_{t+1}^{-1}(\hat{\theta}_t-\theta_{\text{true}}) + \frac{X_t^TP_{t+1}X_t}{(1+e^{y_t\hat{\theta}_t^TX_t})^2} + 2\frac{y_tX_t^T}{1+e^{y_t\hat{\theta}_t^TX_t}}(\hat{\theta}_t-\theta_{\text{true}}),
\end{equation*}
yielding the following equality:
\begin{multline*}
    \sum\limits_{t\in T_{\varepsilon}} \left((\frac{y_tX_t}{1+e^{y_t\hat{\theta}_t^TX_t}})^T(\theta_{\text{true}}-\hat{\theta}_t) - \frac{1}{2}(\hat{\theta}_t-\theta_{\text{true}})^T(P_{t+1}^{-1}-P_t^{-1})(\hat{\theta}_t-\theta_{\text{true}})\right) \\
    = \frac{1}{2} \sum\limits_{t\in T_{\varepsilon}} \frac{X_t^TP_{t+1}X_t}{(1+e^{y_t\hat{\theta}_t^TX_t})^2} + \frac{1}{2} \sum\limits_{t\in T_{\varepsilon}}
    \left((\hat{\theta}_t-\theta_{\text{true}})^TP_t^{-1}(\hat{\theta}_t-\theta_{\text{true}}) - (\hat{\theta}_{t+1}-\theta_{\text{true}})^TP_{t+1}^{-1}(\hat{\theta}_{t+1}-\theta_{\text{true}})\right)\\
    = \frac{1}{2} \sum\limits_{t=1}^n \frac{X_t^TP_{t+1}X_t}{(1+e^{y_t\hat{\theta}_t^TX_t})^2} + \frac{1}{2} \sum\limits_{t=1}^n
    \left((\hat{\theta}_t-\theta_{\text{true}})^TP_t^{-1}(\hat{\theta}_t-\theta_{\text{true}}) - (\hat{\theta}_{t+1}-\theta_{\text{true}})^TP_{t+1}^{-1}(\hat{\theta}_{t+1}-\theta_{\text{true}})\right) \\
    \qquad- \frac{1}{2} \sum\limits_{t\notin T_{\varepsilon}} \frac{X_t^TP_{t+1}X_t}{(1+e^{y_t\hat{\theta}_t^TX_t})^2} - \frac{1}{2} \sum\limits_{t\notin T_{\varepsilon}}
    \left((\hat{\theta}_t-\theta_{\text{true}})^TP_t^{-1}(\hat{\theta}_t-\theta_{\text{true}}) - (\hat{\theta}_{t+1}-\theta_{\text{true}})^TP_{t+1}^{-1}(\hat{\theta}_{t+1}-\theta_{\text{true}})\right).
\end{multline*}
As $\sum\limits_{t\notin T_{\varepsilon}} \frac{X_t^TP_{t+1}X_t}{(1+e^{y_t\hat{\theta}_t^TX_t})^2} \ge 0$, we obtain
\begin{multline*}
    \sum\limits_{t\in T_{\varepsilon}} \left((\frac{y_tX_t}{1+e^{y_t\hat{\theta}_t^TX_t}})^T(\theta_{\text{true}}-\hat{\theta}_t) - \frac{1}{2}(\theta_{\text{true}}-\hat{\theta}_t)^T(P_{t+1}^{-1}-P_t^{-1})(\theta_{\text{true}}-\hat{\theta}_t)\right) \\
    \le \frac{1}{2} \sum\limits_{t=1}^n \frac{X_t^TP_{t+1}X_t}{(1+e^{y_t\hat{\theta}_t^TX_t})^2} + \frac{1}{2} \sum\limits_{t=1}^n
    \left((\hat{\theta}_t-\theta_{\text{true}})^TP_t^{-1}(\hat{\theta}_t-\theta_{\text{true}}) - (\hat{\theta}_{t+1}-\theta_{\text{true}})^TP_{t+1}^{-1}(\hat{\theta}_{t+1}-\theta_{\text{true}})\right) \\
    - \frac{1}{2} \sum\limits_{t\notin T_{\varepsilon}}
    \left((\hat{\theta}_t-\theta_{\text{true}})^TP_t^{-1}(\hat{\theta}_t-\theta_{\text{true}}) - (\hat{\theta}_{t+1}-\theta_{\text{true}})^TP_{t+1}^{-1}(\hat{\theta}_{t+1}-\theta_{\text{true}})\right).
\end{multline*}
Using similar arguments than in the proof of Proposition \ref{cesabianchimodified}, we obtain
$$
    \sum\limits_{t=1}^{n} \frac{X_t^TP_{t+1}X_t}{(1+e^{y_t\hat{\theta}_t^TX_t})^2} \le \frac{(1+e^D)}2d \log\left(1+n p_1D_X^2 \right).
$$
The telescopic sum yields the desired result
\begin{equation*}
    \sum\limits_{t=1}^n
    \left((\hat{\theta}_t-\theta_{\text{true}})^TP_t^{-1}(\hat{\theta}_t-\theta_{\text{true}}) - (\hat{\theta}_{t+1}-\theta_{\text{true}})^TP_{t+1}^{-1}(\hat{\theta}_{t+1}-\theta_{\text{true}})\right) \le \frac{1}{p_1}\|\theta_{\text{true}}\|^2.
\end{equation*}
\end{proof}

\begin{proof}[Proof of Lemma \ref{lemmamartingale}.]
We apply Lemma B.1 of \cite{bercu2008exponential}) on the martingale difference $\Delta M_t\mathds{1}_{| X_t^T(\hat\theta_t-\theta_{\text{true}})|<\varepsilon }$ (as $X_t^T(\hat\theta_t-\theta_{\text{true}})$ is adapted to the filtration $\sigma(X_1,y_1,\ldots,X_{t-1},y_{t-1},X_t)$) in order to obtain
\begin{equation*}
    \mathbb{E}\left[\exp\left(\sum\limits_{t\in T_{\varepsilon}} \left(\lambda \Delta M_t - \frac{\lambda^2}2((\Delta M_t)^2+\mathbb{E}_t[(\Delta M_t)^2])\right)\right) \right]\le 1,\qquad \lambda>0\,.
\end{equation*}
We will prove that 
\begin{equation}\label{eq:quadra}
    (\Delta M_t)^2+\mathbb{E}_t[(\Delta M_t)^2]) \le 2(1+e^D) \frac{(\theta_{\text{true}}-\hat{\theta}_t)^TX_tX_t^T(\theta_{\text{true}}-\hat{\theta}_t)}{(1+e^{\hat{\theta}_t^TX_t})(1+e^{-\hat{\theta}_t^TX_t})}
\end{equation}
in order to achieve
\begin{multline}
    \label{boundmartingale}
    \mathbb{E}\left[\exp\left(\sum\limits_{t\in T_{\varepsilon}} \left(\lambda \Delta M_t - \lambda^2(1+e^D)\frac{(\theta_{\text{true}}-\hat{\theta}_t)^TX_tX_t^T(\theta_{\text{true}}-\hat{\theta}_t)}{(1+e^{\hat{\theta}_t^TX_t})(1+e^{-\hat{\theta}_t^TX_t})}\right)\right)\right] \\
    \le \mathbb{E}\left[\exp\left(\sum\limits_{t\in T_{\varepsilon}} \left(\lambda \Delta M_t - \frac{\lambda^2}2((\Delta M_t)^2+\mathbb{E}_t[(\Delta M_t)^2])\right)\right)\right] \le 1.
\end{multline}
We obtain the inequality \eqref{eq:quadra} by first developing the quadratic term $\mathbb{E}_t[(\Delta M_t)^2])$   as
\begin{align*}
    \mathbb{E}_t[(\Delta M_t)^2]) &= (\theta_{\text{true}}-\hat{\theta}_t)^T\frac{X_tX_t^T}{(1+e^{\hat{\theta}_t^TX_t})(1+e^{-\hat{\theta}_t^TX_t})}(\theta_{\text{true}}-\hat{\theta}_t)\mathbb{E}_t\left[\frac{1+e^{-y_t\hat{\theta}_t^TX_t}}{1+e^{y_t\hat{\theta}_t^TX_t}}\right]\,,\\
    \mathbb{E}_t\left[\frac{1+e^{-y_t\hat{\theta}_t^TX_t}}{1+e^{y_t\hat{\theta}_t^TX_t}}\right] & =  \frac{1+e^{-\hat{\theta}_t^TX_t}}{(1+e^{\hat{\theta}_t^TX_t})(1+e^{-\theta_{\text{true}}^TX_t})} + \frac{1+e^{\hat{\theta}_t^TX_t}}{(1+e^{-\hat{\theta}_t^TX_t})(1+e^{\theta_{\text{true}}^TX_t})}  \\
    & = \frac{a}{1+x} + \frac{a^{-1}}{1+x^{-1}},
\end{align*}
with $a=\frac{1+e^{-\hat{\theta}_t^TX_t}}{1+e^{\hat{\theta}_t^TX_t}}$ and $x=e^{-\theta_{\text{true}}^TX_t}$. As the function  $x\rightarrow \frac{a}{1+x} + \frac{a^{-1}}{1+1/x}$ is monotonic with limits $a$ and $a^{-1}$ in $0$ and $+\infty$, we get $\mathbb{E}_t\left[\frac{1+e^{-y_t\hat{\theta}_t^TX_t}}{1+e^{y_t\hat{\theta}_t^TX_t}}\right] \le \frac{1+e^{D}}{1+e^{-D}}<1+e^{D}$.
To conclude to the inequality \eqref{eq:quadra} we write
\begin{equation*}
    (\Delta M_t)^2 = \frac{(\theta_{\text{true}}-\hat{\theta}_t)^TX_tX_t^T(\theta_{\text{true}}-\hat{\theta}_t)}{(1+e^{\hat{\theta}_t^TX_t})(1+e^{-\hat{\theta}_t^TX_t})} \frac{1+e^{y_t\hat{\theta}_t^TX_t}}{1+e^{-y_t\hat{\theta}_t^TX_t}},
\end{equation*}
and we notice that $\frac{1+e^{y_t\hat{\theta}_t^TX_t}}{1+e^{-y_t\hat{\theta}_t^TX_t}}\le \frac{1+e^{D}}{1+e^{-D}}<1+e^{D}$. Therefore the inequality \eqref{eq:quadra} is proved.\\
Using \eqref{boundmartingale} and the Chernoff's bound, we get for any $\alpha,\gamma >0$ and $\lambda=\alpha/(1+e^D)$,
\begin{align*}
    \mathbb{P}&\left(\sum\limits_{t\in T_{\varepsilon}} \left(\Delta M_t - \alpha\frac{(\theta_{\text{true}}-\hat{\theta}_t)^TX_tX_t^T(\theta_{\text{true}}-\hat{\theta}_t)}{(1+e^{\hat{\theta}_t^TX_t})(1+e^{-\hat{\theta}_t^TX_t})}\right) > \gamma  \right)\\
    &\le \exp(-\lambda \gamma  )\mathbb{E}\left[\exp\left(\sum\limits_{t\in T_{\varepsilon}} \left(\lambda \Delta M_t - \lambda^2(1+e^D)\frac{(\theta_{\text{true}}-\hat{\theta}_t)^TX_tX_t^T(\theta_{\text{true}}-\hat{\theta}_t)}{(1+e^{\hat{\theta}_t^TX_t})(1+e^{-\hat{\theta}_t^TX_t})}\right)\right) \right]\\
    &\le e^{-\frac{\alpha\gamma}{1+e^D}}.
\end{align*}
Setting $\gamma=\frac{1+e^D}{\alpha}\log(\delta^{-1})$ for any $\delta>0$ yields the result.
\end{proof}

\begin{proof}[Proof of Lemma \ref{lemmarecursion}.]
We use power functions of $\|\hat{\theta}_{t+1}-\theta_{\text{true}}\|^2$ identified as
\begin{equation*}
    (\hat{\theta}_{t+1}-\theta_{\text{true}})^T(\hat{\theta}_{t+1}-\theta_{\text{true}}) = (\hat{\theta}_t-\theta_{\text{true}})^T(\hat{\theta}_t-\theta_{\text{true}}) + 2 \frac{(\hat{\theta}_t-\theta_{\text{true}})^TP_{t+1}y_tX_t}{1+e^{y_t\hat{\theta}_t^TX_t}} + \frac{X_t^TP_{t+1}^2X_t}{(1+e^{y_t\hat{\theta}_t^TX_t})^2}.
\end{equation*}
Developing the power function of order $k$, we obtain
\begin{multline*}
    \|\hat{\theta}_{t+1}-\theta_{\text{true}}\|^{2k} = \left(\|\hat{\theta}_t-\theta_{\text{true}}\|^{2} + 2\frac{(\hat{\theta}_t-\theta_{\text{true}})^TP_{t+1}y_tX_t}{1+e^{y_t\hat{\theta}_t^TX_t}}\right)^k\\
    + \frac{X_t^TP_{t+1}^2X_t}{(1+e^{y_t\hat{\theta}_t^TX_t})^2} \sum\limits_{i=1}^k \binom{k}{i} \left(\|\hat{\theta}_t-\theta_{\text{true}}\|^{2} + 2\frac{(\hat{\theta}_t-\theta_{\text{true}})^TP_{t+1}y_tX_t}{1+e^{y_t\hat{\theta}_t^TX_t}}\right)^{k-i}\left(\frac{X_t^TP_{t+1}^2X_t}{(1+e^{y_t\hat{\theta}_t^TX_t})^2}\right)^{i-1}.
\end{multline*}
By definition we note that $P_{t+1}\preccurlyeq P_1 = p_1 I$ so that
\begin{equation*}
    \left| \|\hat{\theta}_t-\theta_{\text{true}}\|^{2} + 2\frac{(\hat{\theta}_t-\theta_{\text{true}})^TP_{t+1}y_tX_t}{1+e^{y_t\hat{\theta}_t^TX_t}} \right| \le \|\hat{\theta}_t-\theta_{\text{true}}\|^{2} + 2p_1\|\hat{\theta}_t-\theta_{\text{true}}\|\|X_t\| \le 4D_{\theta}^2 + 2p_1D_{\theta}D_X,
\end{equation*}
and
\begin{equation*}
    \frac{X_t^TP_{t+1}^2X_t}{(1+e^{y_t\hat{\theta}_t^TX_t})^2} \le p_1^2 D_X^2.
\end{equation*}
Assumption \ref{ass2} gives the rate
\begin{equation*}
    \mathbb{E}\left[\frac{X_t^TP_{t+1}^2X_t}{(1+e^{y_t\hat{\theta}_t^TX_t})^2} \right] \le  \frac{M_2}{4 t^{2}}.
\end{equation*}
Summing those bounds in the binomial expansion of order $k$, we get
\begin{multline*}
    \mathbb{E}\left[\frac{X_t^TP_{t+1}^2X_t}{(1+e^{y_t\hat{\theta}_t^TX_t})^2} \sum\limits_{i=1}^k \binom{k}{i} \left(\|\hat{\theta}_t-\theta_{\text{true}}\|^{2} + 2\frac{(\hat{\theta}_t-\theta_{\text{true}})^TP_{t+1}y_tX_t}{1+e^{y_t\hat{\theta}_t^TX_t}}\right)^{k-i} \left(\frac{X_t^TP_{t+1}^2X_t}{(1+e^{y_t\hat{\theta}_t^TX_t})^2}\right)^{i-1}\right] \\
    \le \frac{M_2}{4t^{2}}\frac{1}{p_1^2D_X^2} \sum\limits_{i=1}^k \binom{k}{i} \left(4 D_{\theta}^2 + 2p_1D_{\theta}D_X\right)^{k-i}\left(p_1^2D_X^2\right)^i \le \frac{b_{k,1}}{t^{2}},
\end{multline*}
with $b_{k,1}= \frac{M_2}{4 p_1^2D_X^2}\left(4D_{\theta}^2 + 2p_1D_{\theta}D_X + p_1^2D_X^2\right)^k$. Similarly, we use again the binomial expansion of order $k$ in order to obtain
\begin{multline*}
    \left(\|\hat{\theta}_t-\theta_{\text{true}}\|^{2} + 2\frac{(\hat{\theta}_t-\theta_{\text{true}})^TP_{t+1}y_tX_t}{1+e^{y_t\hat{\theta}_t^TX_t}}\right)^k = \|\hat{\theta}_t-\theta_{\text{true}}\|^{2k} + 2k\, \|\hat{\theta}_t-\theta_{\text{true}}\|^{2(k-1)}  \frac{(\hat{\theta}_t-\theta_{\text{true}})^TP_{t+1}y_tX_t}{1+e^{y_t\hat{\theta}_t^TX_t}} \\
    + \left(2\frac{(\hat{\theta}_t-\theta_{\text{true}})^TP_{t+1}y_tX_t}{1+e^{y_t\hat{\theta}_t^TX_t}}\right)^2 \sum\limits_{i=2}^k \binom{k}{i} \|\hat{\theta}_t-\theta_{\text{true}}\|^{2(k-i)}  \left(2\frac{(\hat{\theta}_t-\theta_{\text{true}})^TP_{t+1}y_tX_t}{1+e^{y_t\hat{\theta}_t^TX_t}}\right)^{i-2}.
\end{multline*}
We use again the elementary bound
\begin{equation*}
    \left| 2\frac{(\hat{\theta}_t-\theta_{\text{true}})^TP_{t+1}y_tX_t}{1+e^{y_t\hat{\theta}_t^TX_t}} \right| \le 2p_1D_{\theta}D_X,
\end{equation*}
and the following estimate which holds under Assumption \ref{ass2}
$$
    \mathbb{E}\left[\left(2\frac{(\hat{\theta}_t-\theta_{\text{true}})^TP_{t+1}y_tX_t}{1+e^{y_t\hat{\theta}_t^TX_t}}\right)^2\right] \le (4D_{\theta})^2 \mathbb{E}\left[X_t^TP_{t+1}^2X_t\right]     \le (4D_{\theta})^2 \frac{M_2  }{t^{2}}.
$$
Summing the terms in the binomial expansion of order $k$ we get
\begin{equation*}
    \mathbb{E}\left[\left(2\frac{(\hat{\theta}_t-\theta_{\text{true}})^TP_{t+1}y_tX_t}{1+e^{y_t\hat{\theta}_t^TX_t}}\right)^2 \sum\limits_{i=2}^k \binom{k}{i} \|\hat{\theta}_t-\theta_{\text{true}}\|^{2(k-i)}\left(2\frac{(\hat{\theta}_t-\theta_{\text{true}})^TP_{t+1}y_tX_t}{1+e^{y_t\hat{\theta}_t^TX_t}}\right)^{i-2}\right]
    \le \frac{b_{2,k}}{t^{2}},
\end{equation*}
with
\begin{align*}
    b_{2,k}=(4D_{\theta})^2 M_2 \frac{1}{(2p_1D_{\theta}D_X)^2}(4D_{\theta}^2+2p_1D_{\theta}D_X)^k  =  \frac{4M_2}{p_1^2D_X^2}(4D_{\theta}^2+2p_1D_{\theta}D_X)^k.
\end{align*}
Hence we have
\begin{equation*}
    \mathbb{E}\left[\|\hat{\theta}_{t+1}-\theta_{\text{true}}\|^{2k} \right] \le \mathbb{E}\left[\|\hat{\theta}_t-\theta_{\text{true}}\|^{2k}\right] + 2k\mathbb{E}\left[ \|\hat{\theta}_t-\theta_{\text{true}}\|^{2(k-1)}  \frac{(\hat{\theta}_t-\theta_{\text{true}})^TP_{t+1}y_tX_t}{1+e^{y_t\hat{\theta}_t^TX_t}}\right] + \frac{b_k}{t^{2}},
\end{equation*}
with $b_k\ge b_{1,k}+b_{2,k}$. We then apply Proposition \ref{propertyexpectedbound} deriving
\begin{equation*}
    \mathbb{E}_t\left[\frac{y_t}{1+e^{y_t\hat{\theta}_t^TX_t}} \right] = - \frac{X_t^T(\hat{\theta}_t-\theta_{\text{true}})}{(1+e^{\hat{\theta}_t^TX_t})(1+e^{-\hat{\theta}_t^TX_t})}c_t\,,
\end{equation*}
with $e^{-D}<c_t <e^D$ and the tower property in order to obtain
\begin{multline*}
    \mathbb{E}\left[\|\hat{\theta}_{t+1}-\theta_{\text{true}}\|^{2k} \right] \le \mathbb{E}\left[\|\hat{\theta}_t-\theta_{\text{true}}\|^{2k}\right]\\
    - \frac{kc_t}{1+e^D} \mathbb{E}\left[\|\hat{\theta}_t-\theta_{\text{true}}\|^{2(k-1)}(\hat{\theta}_t-\theta_{\text{true}})^T\left( P_{t+1}X_tX_t^T \right)(\hat{\theta}_t-\theta_{\text{true}})\right] + \frac{b_k}{t^{2}}\,.
\end{multline*}
Then Assumption \ref{ass1} applied thanks to the tower property yields
\begin{equation*}
    \mathbb{E}\left[\|\hat{\theta}_{t+1}-\theta_{\text{true}}\|^{2k}\right] \le \mathbb{E}\left[\|\hat{\theta}_t-\theta_{\text{true}}\|^{2k}\right]\left(1-\frac{e^{-D}k  m_1}{t(1+e^D)}\right) + \frac{b_k}{t^{2}}.
\end{equation*}
\end{proof}

\begin{proof}[Proof of Corollary \ref{boundpowerk}.]
Defining $l_t=\mathbb{E}\left[\|\hat{\theta}_t-\theta_{\text{true}}\|^{2k}\right]$ and according to Lemma \ref{lemmarecursion} we have the inequality
\begin{equation*}
    l_{t+1} \le l_t\Big(1-\frac{ka}{t}\Big) + \frac{b_k}{t^{2}},\qquad t\ge 1\,.
\end{equation*}
By a recursive argument it yields to the estimate for $t\ge 2$
\begin{equation*}
    l_{t} \le \sum\limits_{\tau=1}^{t-1}\frac{b_k}{\tau^{2}}\prod\limits_{s=\tau+1}^{t-1}\Big(1-\frac{ka}{s}\Big) + l_1\prod\limits_{s=1}^{t-1}\Big(1-\frac{ka}{s}\Big) \le \sum\limits_{\tau=1}^{t-1}\frac{b_k}{\tau^{2}}\prod\limits_{s=\tau+1}^{t-1}\Big(1-\frac{ka}{s}\Big),
\end{equation*}
because $l_1>0$, $1-ka<0$ and for $s>1$, $1-\frac{ka}{s}>0$. Moreover, taking the logarithm of the products, we estimate
\begin{equation*}
    \sum\limits_{s=\tau+1}^{t-1}\log\Big(1-\frac{ka}{s}\Big) \le -ka \sum\limits_{s=\tau+1}^{t-1}\frac{1}{s} \le -ka \int\limits_{\tau+1}^t \frac{du}{u} = ka\left(\log(\tau+1)-\log(t)\right)\,.
\end{equation*}
It provides the bound
\begin{equation*}
    \prod\limits_{s=\tau+1}^{t-1}(1-\frac{ka}{s}) \le \frac{(\tau+1)^{ka}}{t^{ka}},
\end{equation*}
yielding the estimate
\begin{equation*}
    l_t \le \frac{b_k}{t^{ka}} \sum\limits_{\tau=1}^{t-1} (\tau+1)^{ka-2}\left(\frac{\tau+1}{\tau}\right)^2 \le \frac{4b_k}{t^{ka}} \sum\limits_{\tau=1}^{t-1} (\tau+1)^{ka-2}.
\end{equation*}
As $-1<ka-2<0$, we infer that
\begin{equation*}
    \sum\limits_{\tau=1}^{t-1} (\tau+1)^{ka-2} \le \int\limits_1^t u^{ka-2}du = \frac{1}{ka-1}(t^{ka-1}-1),
\end{equation*}
so that
\begin{equation*}
    l_{t} \le \frac{4b_k}{ka-1}(\frac{1}t -\frac{1}{t^{ka}}),
\end{equation*}
and Corollary \ref{boundpowerk} follows for $t\ge 2$.
\end{proof}

\end{document}